\documentclass[10pt,journal,compsoc]{IEEEtran}
\usepackage{amsmath,amsfonts}
\usepackage{algorithmic}
\usepackage{array}
\usepackage[caption=false,font=normalsize,labelfont=sf,textfont=sf]{subfig}
\usepackage{textcomp}
\usepackage{stfloats}
\usepackage{url}
\usepackage{verbatim}
\usepackage{graphicx}
\usepackage{subcaption}
\def\BibTeX{{\rm B\kern-.05em{\sc i\kern-.025em b}\kern-.08em
    T\kern-.1667em\lower.7ex\hbox{E}\kern-.125emX}}
\usepackage{balance}
\usepackage{xcolor}
\usepackage{pifont}
\usepackage{adjustbox}
\usepackage{diagbox}
\usepackage{multirow}
\usepackage{makecell}
\usepackage{amssymb}
\usepackage{amsthm}
\theoremstyle{plain}
\newtheorem{theorem}{Theorem}

\theoremstyle{remark}

\theoremstyle{definition}

\definecolor{myred}{RGB}{255,0,0}

\usepackage{amsmath,amsfonts,bm}

\def\eqref#1{equation~\ref{#1}}

\def\1{\bm{1}}

\DeclareMathAlphabet{\mathsfit}{\encodingdefault}{\sfdefault}{m}{sl}
\SetMathAlphabet{\mathsfit}{bold}{\encodingdefault}{\sfdefault}{bx}{n}

\usepackage{colortbl}
\usepackage{xcolor}
\definecolor{mygray}{gray}{.9}
\definecolor{goldenrod}{RGB}{245,245,220}
\newlength\savewidth
\newcolumntype{a}{>{\columncolor{mygray}}c}
\usepackage{fontawesome5}
\usepackage{booktabs}
\usepackage{makecell}
\usepackage{fontawesome5}
\usepackage{lipsum}
\usepackage{comment}
\usepackage{multirow}
\usepackage{wrapfig}
\usepackage{algorithm}
\usepackage{algorithmic}
\usepackage{xfrac}  

\usepackage{graphicx}
\usepackage{color}

\usepackage{amsthm}
\definecolor{darkgreen}{rgb}{0,0.7,0}

\definecolor{mygraytext}{gray}{.75}

\begin{document}

\title{MoASE++: Mixture of Activation Sparsity Experts with Domain-Adaptive On-policy Distillation for Continual Test Time Adaptation}

\author{Ronyu Zhang, Aosong Cheng, Gaole Dai, Yulin Luo, Jiaming Liu,\\ Li Du, Huanrui Yang, Dan Wang, Leyuan Fang, Yuan Du, Shanghang Zhang
 \IEEEcompsocitemizethanks{
    \vspace{-1mm}
    \IEEEcompsocthanksitem This manuscript is an extended version of AAAI-26 oral paper~\cite{zhang2026decomposing} (Proc. AAAI, 40(42):36057–36065, 2026; doi:10.1609/aaai.v40i42.40922). \protect\\
}
 \IEEEcompsocitemizethanks{
    \vspace{-1mm}
    \IEEEcompsocthanksitem Rongyu Zhang is with Nanjing University and The Hong Kong Polytechnic University. Aosong Cheng, Gaole Dai, Jiaming Liu, Yulin Luo, and Shanghang Zhang are with the State Key Laboratory of Multimedia Information Processing, School of Computer Science, Peking University. Huanrui Yang is with the University of Arizona. Dan Wang is with the Hong Kong University of Science and Technology. Leyuan Fang is with the School of Artificial Intelligence and Robotics, Hunan University. Yuan Du and Li Du are with Nanjing University. \protect\\
}
}

\markboth{Journal of \LaTeX\ Class Files,~Vol.~14, No.~8, August~2021}%
{Shell \MakeLowercase{\textit{et al.}}: A Sample Article Using IEEEtran.cls for IEEE Journals}

\IEEEtitleabstractindextext{
\begin{abstract}
Continual test-time adaptation adapts a source-pretrained model to non-stationary, unlabeled target streams while retaining past competence, yet texture-biased backbones risk error accumulation and catastrophic forgetting. Drawing inspiration from the process of decoupling shape and texture in the human visual system, we introduce MoASE, a plug-in mixture-of-experts that disentangles domain-agnostic structure from domain-specific texture using Activation Sparsity Experts with Spatial Differentiable Dropout, forming complementary high- and low-activation pathways, while high- and low-rank bottlenecks diversify representations. The Activation Sparsity Gate produces input-adaptive SDD thresholds for precise token selection, and the Domain-Aware Router assigns per-sample expert weights using texture-sensitive cues. To curb confirmation bias on unlabeled streams and stabilize supervision, we then introduce Domain-Adaptive On-Policy Distillation to constitute MoASE++, with an EMA-anchored on-policy reverse KL distillation and an augmentation policy conditioned on entropy and confidence that aligns predictions across the same views and improves the robustness-plasticity balance. Extensive experiments on classification (CIFAR-10/100-C, ImageNet-C) and semantic segmentation (Cityscapes→ACDC) demonstrate consistent state-of-the-art performance, offering a principled, controllable approach to continual adaptation in dynamic visual environments. Code is available on: \url{https://github.com/RoyZry98/MoASE-Pytorch}.
\end{abstract}

}

\maketitle

\begin{figure*}[t]
\centering
\includegraphics[width=0.99\linewidth]{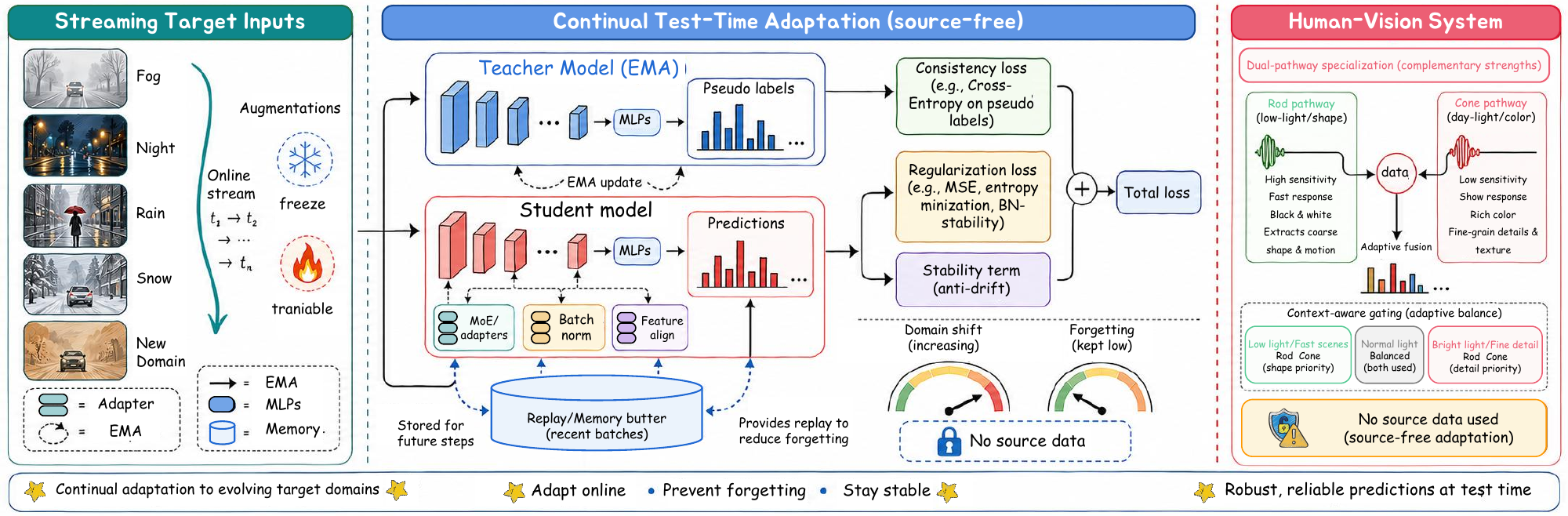}
\caption{\textbf{The problem and motivation.} Our goal is to adapt pre-trained models to evolving target domains.}
\label{fig: teaser}
\end{figure*}

\section{Introduction}
With the emergence of deep-learning-based methods in autonomous driving~\cite{zhao2025survey,zhang2024bevuda++,yu2026identifying} and robotics~\cite{liu2026last,li2023manipllm,ma2026survey}, the continuously changing test-time scenarios in these applications raise significant challenges for stationary machine perception systems~\cite{zhang2024efficient,han2020occuseg,tian2024occ3d}, which anticipates that the test-time data distribution always mirrors that of the training data, resulting in severe error accumulation and catastrophic forgetting. To address this issue, Continual Test-Time Adaptation (CTTA)~\cite{wang2022continual,tian2026dual} has been proposed, which moves beyond the conventional setting of Test-Time Adaptation (TTA)~\cite{wang2020tent,liang2023comprehensive,han2026dota} that handles a single shift to manage a sequence of distribution shifts over time. Current CTTA approaches~\cite{ni2025maintaining,lee2024becotta,liu2023adaptive} predominantly utilize a teacher-student framework to concurrently extract domain knowledge in target domains by generating pseudo-labels. However, such methods frequently struggle to identify and differentiate domain-specific and domain-agnostic features using \texttt{implicit}~\cite{ni2025maintaining,liu2023vida} self-training, visual prompts, and high- and low-rank adapters, which lack interpretability and controllability during training. Moreover, in contrast to implicit models, vision science reveals that the human visual system employs a clearly defined, explicit mechanism with an absolute threshold~\cite{barlow1956retinal,koenig2011absolute} to process visual signals. Therefore, we aim to explore solutions for CTTA tasks from an \texttt{explicit} perspective to decompose feature representations for better perception.

As shown in the right red part of the Fig.~\ref{fig: teaser}. In the human retina, different types of cone cells~\cite{mustafi2009structure} are sensitive to different wavelengths of light (i.e., trichromatic vision). These photoreceptor cells are densely packed in the fovea~\cite{bringmann2018primate} and are sensitive to both texture and hue details, which is essential for fine-grained vision tasks. Another distinct cell type is the rod cell, which is widely distributed across the retina and is responsible for rapid shape and dynamic recognition, although it can only detect light intensity~\cite{von1867handbuch}. Similarly, deep neural networks (DNNs) exhibit parallel characteristics, as shown in previous work \cite{li2024emergence,zhang2024multi,yang2020fda,xu2021fourier}: strongly activated neurons encode domain-agnostic shape and structure, while weakly activated neurons encode domain-specific texture and style. Therefore, we can delineate the conceptual associations between \textbf{\{low/high activation\}-\{domain specific/agnostic\}-\{texture/shape\}.} Drawing inspiration from the human visual system and the prior research, we propose an intriguing hypothesis: $\spadesuit$ \textbf{Can we explicitly decompose neurons by activation degree to encode shapes and textures separately for better model perception to differentiate the continuously changing environments?}

To demystify the role of activation sparsity in CTTA, we manually decompose strongly and weakly activated neurons in a pretrained model and visualize their responses to inputs from varying domains in Fig.~\ref{fig:cam}. We observe a clear distinction in the neuron attention, where \textbf{strongly activated neurons} focus on \textbf{domain-agnostic foreground features} relating to the main object, while \textbf{weakly activated neurons} attend to \textbf{domain-specific background features} of styles and noises. 
This motivates our study on the \texttt{explicit} decomposition of neural activation with a sparse Mixture-of-Experts (MoE)~\cite{jacobs1991adaptive,masoudnia2014mixture,liu2026survey} architecture to mimic the decoupled functionality of cone and rod cells in the Retina~\cite{field2005retinal}. We propose the Mixture-of-Activation-Sparsity-Experts (MoASE), an adapter integrated into pre-trained models that features a Spatial Differentiable Dropout (SDD) mechanism. This setup enhances the extraction of domain-agnostic object shapes and structures and identifies domain-specific textures from a spatial perspective, such as weather-related noise from fog and snow~\cite{zhang2024efficient} through various specialized expert modules.


Moreover, we developed a multi-gate module to enhance dynamic perception in the fluctuating CTTA, comprising the Domain-Aware Router (DAR), which provides domain-specific information to experts, and the Activation Sparsity Gate (ASG), which adjusts the threshold for each expert's activation selection. Concretely, DAR leverages texture-sensitive, low-activation cues to compute dense, per-sample routing weights, enabling soft collaboration among experts rather than brittle, hard assignment, thereby improving robustness to rapid, non-stationary shifts. In parallel, ASG predicts input-adaptive keep ratios for SDD, endowing each expert with a controllable perceptual field that tightens or relaxes token selection on the fly, thereby calibrating the balance between plasticity (domain-specific) and stability (domain-agnostic). This explicit, two-pronged gating makes MoASE not only more accurate but also more interpretable: \textit{routing decisions and activation thresholds are exposed as actionable signals that can be monitored and tuned across time.}

MoASE explicitly performs activation decomposition, yet under continual shifts, it may still accumulate confirmation bias, and its ASG and DAR thresholds and routing can become destabilized. To address these challenges, we extend MoASE to \textbf{MoASE++} with Domain‑Adaptive On‑Policy Distillation (DA-OPD), which provides stable and adaptive supervision under the teacher–student framework on unlabeled target streams. At the core is an EMA‑anchored teacher that \textit{acts as a slow and steady optimization anchor}, and we align the student to this anchor using temperature‑scaled reverse KL on the same adaptively sampled on‑policy augmented view, which mitigates confirmation bias that would otherwise arise from multi‑view mismatch. A lightweight policy network, conditioned on sample entropy and confidence, parameterizes multiple perturbations, while entropy regularization and strength penalties preserve diversity without destructive edits. DA‑OPD complements MoASE by operating at the input level to amplify texture‑sensitive cues and stabilize the teacher, while MoASE disentangles structure and texture at the feature level, together suppressing error accumulation, reducing inter‑domain divergence, and improving the robustness–plasticity balance under dynamic test‑time shifts.

Extensive experiments demonstrate the superiority of our proposed MoASE++ across three image classification benchmarks~\cite{krizhevsky2009learning,hendrycks2019benchmarking} and one segmentation benchmark~\cite{cordts2016cityscapes,sakaridis2021acdc} on CTTA scenarios with improvements exceeding 16.1\% in classification accuracy and 5.8\% in segmentation mIoU.
These results underscore MoASE++'s robust capability to adapt and perform reliably across dynamic environments. 

The major contributions can be summarized as:
\begin{itemize}
    \item We draw inspiration from the human visual system to develop a MoASE model that addresses error accumulation and catastrophic forgetting in the face of continuously changing distributions.
    \item We decompose the neuron-encoded activation into domain-specific and domain-agnostic features, using distinct expert models to encode texture and shape independently with SDD.
    \item We develop a multi-gate module featuring the DAR and the ASG, which utilize domain information to dynamically generate adaptive routing strategies and expert activation thresholds.
    \item We introduce DA-OPD to constitute MoASE++, an EMA-anchored on-policy reverse-KL distillation with an entropy- and confidence-conditioned augmentation policy that reduces confirmation bias and improves the robustness–plasticity balance.
\end{itemize}

\clearpage

\section{Related works}
\subsection{Continual Test-Time Adaptation}
CTTA addresses the challenge of adapting to a non-static target domain, a complication for traditional TTA methods. The pioneering work CoTTA~\cite{wang2022continual} combined bi-average pseudo-labels with stochastic weight resets to address this issue. To mitigate error accumulation, ECoTTA~\cite{song2023ecotta} employed a meta-network for output regularization. While these approaches focused on model-level solutions, other studies~\cite{gan2023decorate,yang2023exploring, ni2023distribution,liu2023vida} explored using visual domain prompts or minimal parameter adjustments for continual learning. For example, Liu~\cite{liu2023adaptive} introduced reconstruction techniques for continual adaptation, and TCA~\cite{ni2025maintaining} addressed domain shifts and error accumulation by enforcing class topological consistency, maintaining intra-class compactness, and leveraging a batch-imbalance topology-weighting mechanism to preserve inter-class stability. Moreover, BECoTTA~\cite{lee2024becotta} incorporated an MoE module into CTTA, thereby promoting effective domain-specific knowledge retention through data augmentation. Unlike previous \texttt{implicit} methods, MoASE++ adopts an \texttt{explicit} approach to CTTA, using controllable activation decomposition via MoE. 

\subsection{Activation Sparsity} 
Activation sparsity refers to the presence of numerous weakly contributing elements in activation outputs~\cite{Chen_2023_CVPR,kurtz2020inducing,yang2019dasnet,yang2019sparse}. SparseViT~\cite{song2024prosparse} revisited this concept for modern window-based vision transformers, aiming to increase speed and reduce computation. Meanwhile, Mirzadeh~\cite{mirzadeh2023relu} explored the reuse of activated neurons in LLMs, proposing strategies to reduce computation. Recently, ProSparse~\cite{song2025prosparse} has achieved high activation sparsity in LLMs while maintaining comparable performance, enabling up to 4.52x inference speedup through progressive sparsity regularization with ReLU activation.
However, all previous work aimed to improve model efficiency until~\cite{zhang2024multi,li2024emergence} revealed the contribution of shape bias to model performance, while MoASE++ leverages this characteristic to mitigate the influence of adaptive domain change.

\subsection{Mixture-of-Experts} 
MoE was initially introduced by Jacobs and Jordan~\cite{jacobs1991adaptive,jordan1994hierarchical}, who used independent modules to boost expressiveness and reduce computational costs. Eigen and Ma~\cite{eigen2013learning,ma2018modeling} evolved it into the MoE layer. In natural language processing, GShard~\cite{gshard} and Switch Transformer~\cite{switch} incorporated MoE with top-1/2 routing to enhance capacity. Fixed routing~\cite{hash,stablemoe} and ST-MoE~\cite{st_moe} aimed to stabilize training. Recent developments~\cite{zou2026flylora,zhao2026each} introduced efficient adapters within MoE, and p-FedMoE~\cite{yi2026pfedmoe} combined MoE with federated learning to achieve model-heterogeneous personalization. In addition, MoE-LLaVA~\cite{lin2026moe} incorporated MoE into the multimodal large language models. MoE has also been extended to embodied AI. For example, MoLe-VLA~\cite{zhang2026mole} employed a spatio-temporal router to enable layer-skipping in the vision-language-action model. Unlike previous works that primarily use MoE for parameter scaling, MoASE++ focuses on managing diverse neuron activations through a multi-expert structure.

\subsection{On-Policy Distillation}
OPD addresses the distribution mismatch between training and inference in autoregressive generation. Instead of learning only from fixed human-written or teacher-generated sequences, OPD lets the student sample from its own evolving policy and queries teacher feedback on these self-generated trajectories. 
Recent studies have adapted this principle to large language models. Generalized Knowledge Distillation~\cite{agarwal2024policy} trained on mixtures of teacher- and student-generated sequences to better match the student-induced distribution. MiniLLM~\cite{gu2024minillm} adopted reverse KL divergence for on-policy distillation, mitigating the mode-covering behavior of conventional forward KL. DistiLLM~\cite{ko2024distillm} further introduced skewed KL objectives to stabilize optimization by interpolating between teacher and student distributions. Beyond distribution matching, reinforcement-learning-based distillation methods~\cite{liu2025language,zhang2025distillation} used reward or preference signals to transfer higher-level reasoning behaviors. \textbf{Unlike prior OPD methods that typically query a static or frozen teacher on student-induced samples, our DA-OPD employs a dynamically evolving EMA teacher}, which provides reverse-KL supervision on adaptive on-policy views while maintaining a stable yet domain-aware optimization anchor.

\begin{figure*}[t]
\centering
\includegraphics[width=0.99\linewidth]{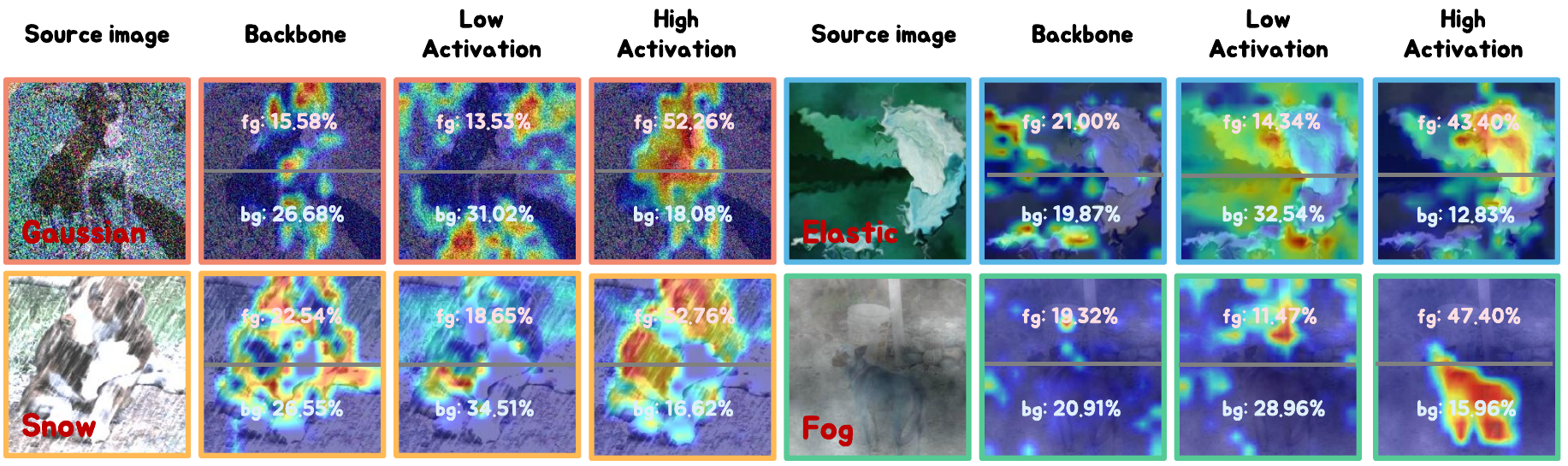}
\caption{\textbf{The visualization analysis of the Class Activation Map (CAM).} We adopt CAM to compare the attention of the low-activation, high-activation MoASE, and the original model during the continual adaptation process. The number suggests the mean IoU of the Elastic-transform, Gaussian, Fog, and Snow results based on randomly selected samples from ImageNet-C with the \textit{Foreground} and \textit{Background} of the main object segmented by SAM‑3.}
\label{fig:cam}
\end{figure*}

\section{Motivation for MoASE}
\label{sec:motivation}
Biological vision distributes computation across parallel, selective pathways. The fovea–periphery organization of the retina, together with downstream circuits, establishes a division of labor: the cone‑dominated fovea supports high‑acuity, color‑sensitive analysis in central vision, whereas the rod‑rich periphery provides high‑sensitivity, wide‑field monitoring that is especially effective under low luminance and motion~\cite{mustafi2009structure, bringmann2018primate, field2005retinal}. Rather than enforcing a rigid foreground–background split, this architecture embodies a general principle: separable streams with distinct selectivity and absolute thresholds collaborate to sustain robust perception across changes in luminance, motion, and clutter~\cite{barlow1956retinal,koenig2011absolute}. We adopt this principle as a design prior for CTTA, construct explicit, parallel pathways with different activation thresholds and specializations, and fuse them adaptively. Therefore, what is preserved (domain‑agnostic structure) and what is adapted (domain‑specific texture) can be controlled and interpreted during continual shifts.

In machine perception, analogous principles emerge in sparse and multi-path representations~\cite{li2024emergence,zhang2024multi}, as well as frequency- and texture-aware robustness~\cite{yang2020fda,xu2021fourier}. Empirically, shape cues tend to be more domain-agnostic than high-frequency textures, which are sensitive to corruption, style, and background statistics. To probe this dichotomy in a CTTA setting, we analyze Class Activation Mapping (CAM)~\cite{zhou2016learning} on ImageNet-to-ImageNet-C with a ViT-Base model~\cite{wang2022continual}. When we retain only high activations, saliency concentrates on object contours and edge intersections across domains, retaining only low activations elevates background fluctuations and domain-specific patterns as shown in Fig.~\ref{fig:cam}. This supports the hypothesis that activation magnitude correlates with a shape–texture split.

In addition, we further quantify this hypothesis in Fig.~\ref{fig:cam} by extracting the main object in the figure with SAM‑3~\cite{carion2025sam} and computing IoU between CAMs and the SAM masks of the domain‑agnostic object and the domain‑specific context. The high‑activation CAMs align much better with the domain‑agnostic object (e.g., Gaussian: 0.5226 vs.\ 0.1808 on context), whereas low‑activation CAMs favor the domain‑specific context (e.g., Elastic: 0.3254 vs.\ 0.1434 on the object). These targeted comparisons substantiate our claim that high activations capture object‑centric, while low activations emphasize context‑centric, motivating the explicit activation decomposition.

These observations motivate an explicit decomposition strategy: we equip a pretrained backbone with MoASE, a mixture-of-experts architecture that applies SDD to retain either the highest or the lowest activations at each spatial location. To avoid fixed and globally optimal thresholds that are unrealistic given the diversity of inputs, we introduce an input-aware ASG that adaptively selects thresholds for each expert. Moreover, since domain identification is often driven by texture and background statistics, we design a DAR that preferentially uses low-activation and texture-sensitive signals to guide expert routing. 

In summary, MoASE operationalizes a biologically inspired yet explicitly engineered principle, parallel selective processing with adaptive gating, to disentangle domain-agnostic structure from domain-specific texture for CTTA. This yields interpretable control over what is adapted, improving robustness across evolving test-time shifts.

\section{Methods}
\label{sec:method}

\subsection{Preliminaries}
\label{sec:prelim}
\textbf{Continual Test-Time Adaptation (CTTA).}
We start from a source-pretrained model and adapt it online to a sequence of unlabeled target domains without accessing source data. Let the source domain be $D_S=(\mathcal{X}_S,\mathcal{Y}_S)$ and the target domains be $\{D_{T_i}\}_{i=1}^{n}$ with feature spaces $\{\mathcal{X}_{T_i}\}_{i=1}^n$. We denote the Student and Teacher models by $\theta^{\mathcal{S}}$ and $\theta^{\mathcal{T}}$, respectively, both initialized from the same source-pretrained parameters. Following mean-teacher practice~\cite{tarvainen2017mean, dobler2023robust}, we update the Teacher via an exponential moving average (EMA) of the Student to stabilize adaptation under continual shifts~\cite{wang2022continual, gan2022decorate}. The adaptation is unsupervised and single-pass on target data $x \in \mathcal{X}_{T_i}$, aiming to maintain performance on previously seen distributions while adapting to evolving ones; see Fig.~\ref{fig:framework} for an overview.

\textbf{Mixture-of-Experts (MoE).}
We consider an MoE composed of $E$ expert functions $\{e_i(\cdot)\}_{i=1}^{E}$ and a trainable routing head $g(\cdot)$ that outputs a probability vector over experts. For an input $x$, the MoE prediction is a convex combination of expert outputs weighted by $g(x)$:
\begin{equation}
\label{eq:moe-routing-unified}
\begin{aligned}
y(x) \;=\; \sum_{i=1}^{E} \phi_i(x)\; e_i(x), 
\\
\boldsymbol{\phi}(x) \;=\; g(x) \;=\; \mathrm{softmax}\!\big(\boldsymbol{A}x + \boldsymbol{b}\big),
\end{aligned}
\end{equation}
subject to the simplex constraints $\phi_i(x) \ge 0$ and $\sum_{i=1}^{E} \phi_i(x) = 1$. Here $\boldsymbol{\phi}(x)\in\mathbb{R}^{E}$ denotes the per-sample expert weights (soft routing), $\boldsymbol{A}\in\mathbb{R}^{E\times d}$ and $\boldsymbol{b}\in\mathbb{R}^{E}$ are trainable parameters, and $d$ is the input dimensionality of the routing head. This gating is dense, assigning nonzero probability mass to all experts by default.

\subsection{Mixture-of-Activation-Sparsity-Experts}
\label{sec:moase}
We introduce a plug-in MoASE++ for continual test-time adaptation that decomposes activations by strength, aiming to separate domain-agnostic structures (high activations) from domain-specific textures (low activations). 

Let $F \in \mathbb{R}^{B \times N \times D}$ be the input feature with batch size $B$, number of tokens $N$, and channel dimension $D$. We consider an even number of experts $E$, partitioned into
$\mathcal{E}_a$ (high-activation preserving, domain-agnostic) and
$\mathcal{E}_s$ (low-activation preserving, domain-specific), with
$|\mathcal{E}_a|=|\mathcal{E}_s|=E/2$.
Each expert is denoted $e_i(\cdot)$ and is paired with a rank-parameterized mapping $f^{(i)}(\cdot)$ as in \eqref{eq:rank-bottleneck}. The token-wise activation score is
$s(b,n)=R(F[b,n,:])$ as in \eqref{eq:token-score}. Each expert $i$ has a base keep-ratio $q_i \in (0,1)$ and a polarity $\square_i \in \{\ge,\le\}$ indicating Top-/Bottom-$K$ selection within SDD \eqref{eq:sdd-mask-unified}.

\begin{figure*}[t]
\centering
\includegraphics[width=0.99\linewidth]{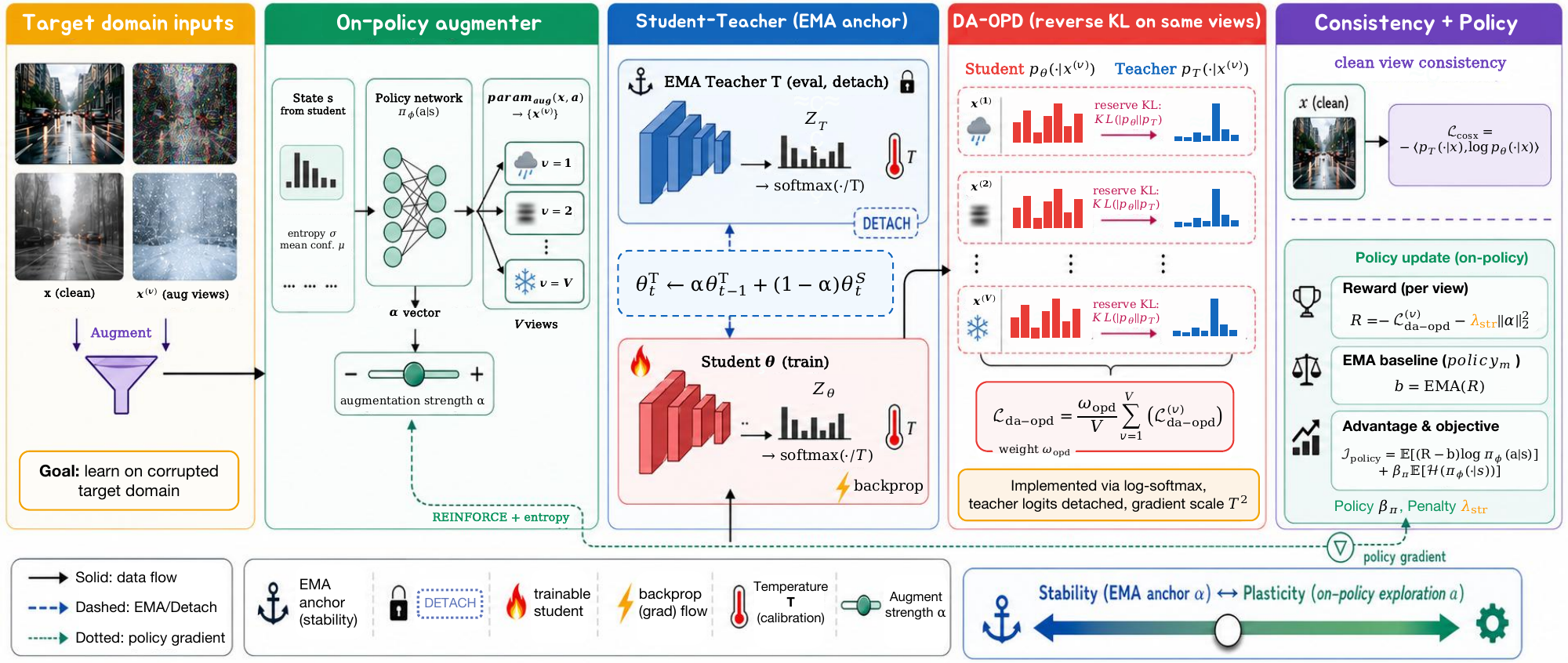}
\caption{\textbf{The framework of MoASE++.} A policy-driven augmenter samples strength a to create V on‑policy views, while the student aligns to a detached EMA teacher via reverse KL on the same views plus a clean‑view consistency term. The policy is updated with a reward, enabling a tunable stability–plasticity trade‑off.} 
\label{fig:framework}
\end{figure*}

\subsubsection{Spatial Differentiable Dropout}

Given an input tensor $F \in \mathbb{R}^{B \times N \times D}$, we first compute token-wise activation scores by reducing along the channel dimension. For a sample index $b \in \{1,\ldots,B\}$ and a token index $n \in \{1,\ldots,N\}$, the score is defined as
\begin{equation}
\begin{aligned}
  s(b,n) &= R\big(F[b,n,:]\big), \\
  R(v)   &= \|v\|_2 \;\text{ or }\; \frac{1}{D}\sum_{d=1}^{D} |v_d| .
\end{aligned}
\label{eq:token-score}
\end{equation}

Each expert $i$ is associated with a base keep-ratio $q_i$ that lies in the open interval $(0,1)$. The corresponding token budget is given by the integer quantity $K_i = \lfloor N \cdot q_i \rfloor$. For every sample $b$ and expert $i$, Spatial Differentiable Dropout constructs a binary mask $M_i \in \{0,1\}^{B \times N}$. This mask is broadcast along the channel dimension to obtain $M_i^{(3\mathrm{D})} \in \{0,1\}^{B \times N \times D}$. The mask is obtained by selecting either the top or bottom $K_i$ tokens based on the scores $s(b,n)$. The selection polarity for expert $i$ is encoded by a symbol $\square_i$ which can be either the relation $\ge$ for top-$K_i$ selection or the relation $\le$ for bottom-$K_i$ selection. The mask definition is
\begin{equation}
\begin{aligned}
      M_i(b,n) &= \mathbb{1}\left[\, s(b,n)\, \square_i \, \tau_i(b) \,\right] ,
  \\
  \tau_i(b) &= \mathrm{K^{th}}\left( \{ s(b,1), s(b,2), \ldots, s(b,N) \} ;\, K_i,\, \square_i \right) ,
\end{aligned}
\label{eq:sdd-mask-unified}
\end{equation}
where $\mathbb{1}$ is the indicator function, the threshold $\tau_i(b)$ denotes the $K_i^{th}$ order statistic of the multiset of scores for sample $b$, computed in descending order when $\square_i$ equals $\ge$ and in ascending order when $\square_i$ equals $\le$.

To equip different experts with heterogeneous representational capacities while preserving the unified SDD masking rule, each expert $i$ applies a rank-parameterized linear bottleneck to the input features. The bottleneck acts token-wise and is defined as
\begin{equation}
\begin{aligned}
      f^{(i)}(b,n,:) \;=\; W^{(i)}_{\downarrow} \!\left( W^{(i)}_{\uparrow}\, F(b,n,:) \right) ,
  \\
  W^{(i)}_{\uparrow} \in \mathbb{R}^{D \times r_i} ,
  \quad
  W^{(i)}_{\downarrow} \in \mathbb{R}^{r_i \times D} .
\end{aligned}
\label{eq:rank-bottleneck}
\end{equation}

The integer $r_i$ is the middle dimension that specifies the rank or capacity of expert $i$. A larger middle dimension, for example $r_i \ge D$, corresponds to a higher-capacity mapping. A smaller middle dimension, for example $r_i \ll D$, corresponds to a compact mapping. After applying the bottleneck mapping, the sparsified input to expert $i$ is given by
\begin{equation}
  \widetilde{F}_i \;=\; M_i^{(3\mathrm{D})} \odot f^{(i)}(F) ,
\label{eq:sdd-masked-ranked}
\end{equation}
where the operator $\odot$ denotes elementwise multiplication and the mapping $f^{(i)}(F)$ is applied independently to every token position.

In practical configurations, the sets of keep-ratios $\{q_i\}$ and middle dimensions $\{r_i\}$ are chosen to span a range of token retention levels and representational capacities across the experts. The selection polarity set $\{\square_i\}$ determines whether an expert focuses on tokens with high or low activation scores. The polarity only affects the mask construction described in \eqref{eq:sdd-mask-unified} and does not change the linear bottleneck mapping in \eqref{eq:rank-bottleneck}. During training, one can employ straight-through estimators or differentiable relaxations of the top-$K$ and bottom-$K$ operators, allowing gradients to propagate through the discrete selection induced by \eqref{eq:sdd-mask-unified}.

\subsubsection{Domain-Aware Router and Activation Sparsity Gate}
To accommodate the dynamic nature of CTTA, we introduce two input-aware modules that provide domain cues and disentangle domain-agnostic structures from domain-specific textures across experts. The lightweight heads are denoted $g_{\mathrm{DAR}}(\cdot)$ and $g_{\mathrm{ASG}}(\cdot)$, and their outputs follow the unified symbol conventions.

\textbf{Domain-Aware Router (DAR).} DAR serves as the routing head and outputs per-sample expert weights, explicitly attending to low-activation content. Given input features $F \in \mathbb{R}^{B \times N \times D}$ and scores $s(b,n)$, as shown in ~\eqref{eq:token-score}, we construct a bottom-activation mask via the unified SDD rule with ~\eqref{eq:sdd-mask-unified} with polarity $\square{=}\le$. The threshold is manually set for SDD as $q=\{\frac{id_{h}}{E},...,\frac{1}{2},\frac{id_{l}}{E},...,\frac{1}{2}\}$, where $id_{h}$ and $id_{l}$ suggests the expert ID for domain-agnostic experts and domain-specific experts. Let $M_{\mathrm{low}} \in \{0,1\}^{B \times N}$ with $K_{\mathrm{low}}=\lfloor N \cdot \tfrac{1}{2} \rfloor$, and broadcast it to $M_{\mathrm{low}}^{(3\mathrm{D})} \in \{0,1\}^{B \times N \times D}$. The low-activation sub-features and routing output are
\begin{equation}
\label{eq:dar-low}
\begin{aligned}
F_{\mathrm{low}} \;=\; M_{\mathrm{low}}^{(3\mathrm{D})} \odot F, \qquad
\phi \;=\; g_{\mathrm{DAR}}(F_{\mathrm{low}}),
\end{aligned}
\end{equation}
where $\phi \in \mathbb{R}^{B \times E}$ satisfies $\phi(b,i)\!\ge\!0$ and $\sum_{i=1}^{E}\phi(b,i)=1$ for each sample $b$, i.e., a probability-simplex routing vector over experts.

\textbf{Activation Sparsity Gate (ASG).} ASG adjusts each expert’s keep-ratio in an input-adaptive manner. It consumes the full features and outputs per-sample offsets:
\begin{equation}
\label{eq:asg-head}
\begin{aligned}
\varepsilon \;=\; g_{\mathrm{ASG}}(F), \qquad \varepsilon \in \mathbb{R}^{B \times E}.
\end{aligned}
\end{equation}

For sample $b$ and expert $i$, the adjusted keep-ratio and token budget are
\begin{equation}
\label{eq:asg-keep}
\begin{aligned}
\widehat{q}_i(b) &= \mathrm{clip}\!\big(q_i + \eta\,\varepsilon(b,i),\; q_{\min},\; q_{\max}\big), \\
\widehat{K}_i(b) &= \big\lfloor N \cdot \widehat{q}_i(b) \big\rfloor,
\end{aligned}
\end{equation}
with a small scaling factor $\eta$ (e.g., 0.1) ensuring $\widehat{q}_i(b)\!\in\![q_{\min}, q_{\max}]$.

Each expert then applies the unified Top/Bottom-$\widehat{K}_i(b)$ rule in ~\eqref{eq:sdd-mask-unified} with its polarity $\square_i \in \{\ge,\le\}$ to construct $M_i$, forming the sparsified inputs
$\widetilde{F}_i = M_i^{(3\mathrm{D})} \odot f^{(i)}(F)$ as shown in ~\eqref{eq:sdd-masked-ranked}. Finally, the MoASE++ aggregation uses DAR’s output $\phi$ as expert weights:
\begin{equation}
\label{eq:dar-agg}
\begin{aligned}
Y(b) \;=\; \sum_{i=1}^{E} \phi(b,i)\; e_i\!\big(\widetilde{F}_i(b)\big).
\end{aligned}
\end{equation}

\begin{figure}[t]
\centering
\includegraphics[width=0.99\linewidth]{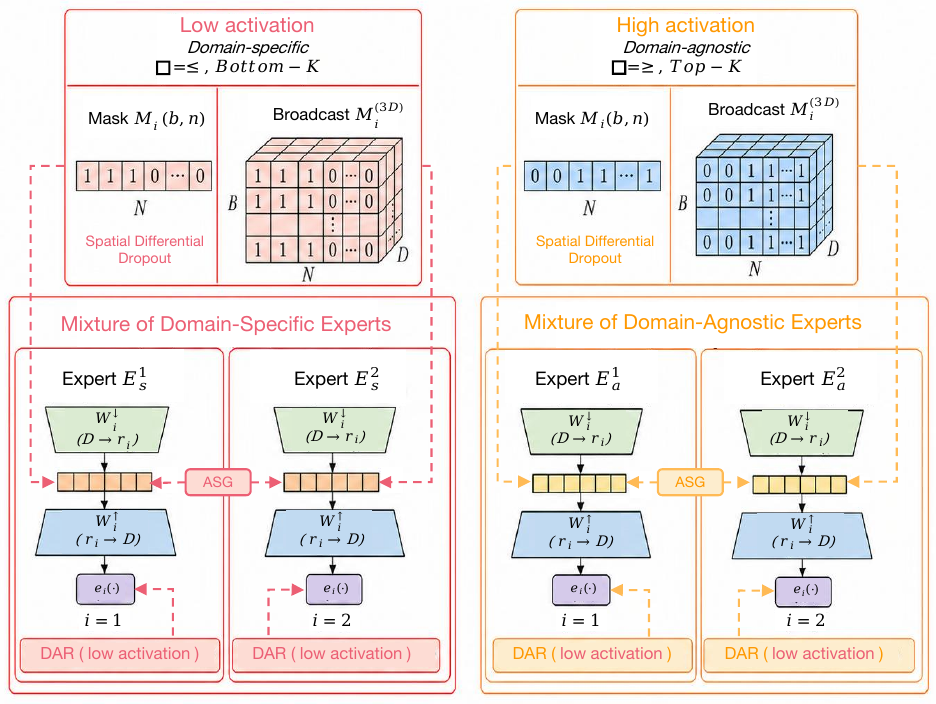}
\caption{\textbf{The detailed structure of MoASE.} We integrate the MoASE into the linear layer of the student model.} 
\label{fig:moase}
\end{figure}

\section{Domain-Adaptive On-Policy Distillation}
\label{sec:daopd}

The central design choice is an EMA-anchored teacher that serves as a slow, stable reference for aligning the student on its own induced distribution. By explicitly treating the EMA teacher as an optimization anchor, DA-OPD replaces fragile entropy-minimization heuristics with a principled reverse Kullback–Leibler (KL) alignment on the same, student-induced augmented views. A lightweight policy network parameterizes augmentation strength, yielding a controllable trade-off between stability (anchored by the EMA teacher) and plasticity (driven by on-policy exploration).

\subsection{Teacher–Student EMA}
\label{subsec:ema}

Following CoTTA, we adopt a teacher–student framework. Let $\theta_t^{S}$ and $\theta_t^{T}$ denote student and teacher parameters at step $t$, respectively. Both models share the same initialization; the teacher is then updated by the exponential moving average (EMA) of the student:
\begin{equation}
\theta_t^{T} \leftarrow \alpha\,\theta_{t-1}^{T} + (1-\alpha)\,\theta_t^{S}, \qquad \alpha\in[0,1).
\label{eq:ema}
\end{equation}

In practice, the teacher remains in evaluation mode and is detached from backpropagation, with momentum $\alpha$. The EMA teacher provides a temporally low-pass-filtered, slowly varying target that dampens target jitter and variance, mitigates catastrophic forgetting under continual shifts, and serves as a stable but also dynamic anchor for the DA-OPD objective.

\subsection{On-Policy Distillation with EMA Anchored}
\label{subsec:opd_mech}

Given an unlabeled target-domain input $x$, we form on-policy augmented views $\{x^{(v)}\}_{v=1}^{V}$ under the current student. When RL-OPD is enabled, a policy network $\pi_\phi(a\mid s)$ consumes state statistics derived from the student on the clean view (e.g., batch entropy and mean confidence) and outputs a continuous augmentation-strength vector $a$ used by a parameterized augmenter $\mathrm{param\_augment}(x,a)$. Otherwise, we use a fixed stochastic augmentation pipeline. The student and the EMA teacher produce logits on the same view $x^{(v)}$ with temperature $T>0$, we define
\begin{align}
p_\theta(\cdot\mid x^{(v)}) &= \mathrm{softmax}\left(\frac{z_\theta}{T}\right), \\
p_T(\cdot\mid x^{(v)}) &= \mathrm{softmax}\left(\frac{z_T}{T}\right).
\label{eq:tempreture}
\end{align}

DA-OPD minimizes the reverse KL on each view,
\begin{equation}
\mathcal{L}^{(v)}_{\mathrm{DA\text{-}OPD}}
= \mathrm{KL}\!\left(p_\theta(\cdot\mid x^{(v)})\,\|\,p_T(\cdot\mid x^{(v)})\right),
\label{eq:daopd_kl}
\end{equation}
implemented via $\log\mathrm{softmax}$ for numerical stability with teacher logits detached. To match temperature scaling, gradients are corrected by a factor $T^2$.

On the original (clean) view $x$, we retain a base consistency term aligning the student to the EMA teacher:
\begin{equation}
\mathcal{L}_{\mathrm{cons}}(x)
= -\, \big\langle p_T(\cdot\mid x),\ \log p_\theta(\cdot\mid x)\big\rangle,
\label{eq:cons}
\end{equation}
where $\log p_\theta$ is obtained from the student’s $\log\mathrm{softmax}$. Across augmented views, the DA-OPD term is averaged with weight $\omega_{\mathrm{opd}}$,
\begin{equation}
\begin{aligned}
\mathcal{L}_{\mathrm{da\text{-}opd}}\big(\{x^{(v)}\}\big) &=
\frac{\omega_{\mathrm{opd}}}{V}
\sum_{v=1}^{V}\mathcal{L}^{(v)}_{\mathrm{DA\text{-}OPD}}, 
\end{aligned}
\label{eq:daopd_avg}
\end{equation}

With RL-OPD active, policy parameters $\phi$ are updated by reinforcement learning using a reward that favors smaller reverse KL while penalizing excessive augmentation strength:
\begin{equation}
R = -\,\mathcal{L}^{(v)}_{\mathrm{DA\text{-}OPD}} - \lambda_{\mathrm{str}}\|a\|_2^2.
\label{eq:reward}
\end{equation}

We maintain an EMA baseline $b$ of rewards. The policy objective combines advantage-weighted log-likelihood with entropy regularization:
\begin{equation}
\begin{aligned}
\mathcal{J}_{\mathrm{policy}} & =
\mathbb{E}_{a\sim\pi_\phi(\cdot\mid s)}\!\big[(R-b)\,\log\pi_\phi(a\mid s)\big] \\
& \quad + \beta_\pi\,\mathbb{E}\big[\mathcal{H}(\pi_\phi(\cdot\mid s))\big], 
\\  \beta_\pi & =\texttt{policy\_beta}.
\end{aligned}
\label{eq:policy}
\end{equation}

In practice, policy updates are performed using computation graphs with student-teacher detachment.

\subsection{Overall Objective and Update Schedule}
\label{subsec:overall}

At each step, we solve a saddle-point problem for the student and the policy:
\begin{equation}
\begin{aligned}
\min_{\theta} \  
\Big\{ \ \mathcal{L}_{\mathrm{cons}}(x; \theta, \theta^T) 
& + \mathcal{L}_{\mathrm{da\text{-}opd}}(\{x^{(v)}\}; \theta, \theta^T) \ \Big\}, \\
\max_{\phi} \ & \mathcal{J}_{\mathrm{policy}}(s; \phi).
\end{aligned}
\label{eq:minmax}
\end{equation}

The update schedule is student $\rightarrow$ policy $\rightarrow$ teacher: first backpropagate $\mathcal{L}_{\mathrm{cons}}+\mathcal{L}_{\mathrm{da\text{-}opd}}$ to update the student. Then update the policy using the detached reward and baseline. Finally, update the teacher via \eqref{eq:ema}. To curb drift, we apply stochastic weight restoration and keep the teacher in evaluation mode to preserve the anchor’s smoothness.

By anchoring on an EMA teacher and aligning with student-induced views, DA-OPD avoids off-policy mismatch and suppresses amplification of pseudo-label noise. The parameterized augmenter grants fine-grained control over exploration versus conservatism, while temperature, view count, and alignment weight modulate coupling strength to the EMA anchor. Together with MoASE, DA-OPD constitutes MoASE++, delivering markedly improved robustness and interpretability under continual test-time domain shifts with computational overhead comparable to a single extra teacher forward pass.

\begin{figure*}[t]
\centering
\includegraphics[width=0.99\linewidth]{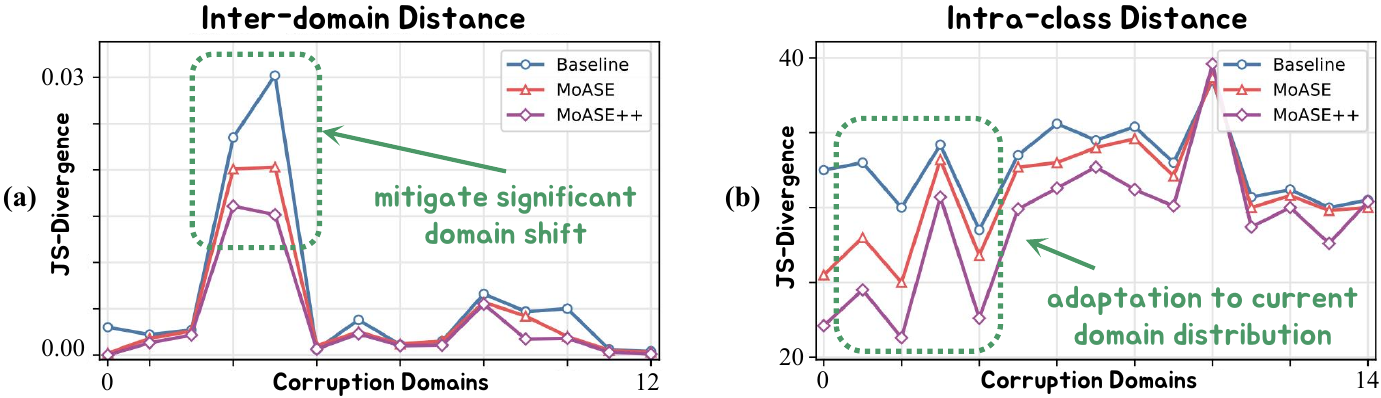}
\caption{\textbf{Inter-domain and intra-class distance in CIFAR10-C.} (a) MoASE++ and MoASE more effectively reduce inter-domain divergence than the source model. (b) MoASE++ and MoASE significantly improve intra-class feature aggregation, producing results that closely align with those of our proposed method.}
\label{divergence}
\end{figure*}

\section{Justification}
To substantiate our hypothesis, we establish a theoretical bound on the target domain error, which quantifies the generalization performance of a hypothesis $h_y$ trained on the source domain $D_S$ when applied to the target domain $D_T$. This analysis builds on the domain transfer theory proposed by Ben-David et al.~\cite{ben2006analysis, ben2010theory}. We formalize this as the following theorem:

\begin{theorem}
Let $h_y$ be a hypothesis in the hypothesis space $\mathcal{H}$, and let $f_S$ and $f_T$ denote the labeling functions for the source and target domains, respectively. Then, the target domain error $\epsilon_T(h_y)$ can be bounded as:
\begin{equation}
\begin{aligned}
    \epsilon_T(h_y) &\leq \epsilon_S(h_y) + d(D_S, D_T) + \\ &\min \big\{ E_{D_S}[|f_S(\mathbf{x}) - f_T(\mathbf{x})|], E_{D_T}[|f_S(\mathbf{x}) - f_T(\mathbf{x})|] \big\},
\end{aligned}
\label{eq:bound}
\end{equation}
where $\epsilon_S(h_y)$ is the error on the source domain, $d(D_S, D_T)$ is the total variation divergence between the source and target distributions, and the final term quantifies the discrepancy between the labeling functions.
\end{theorem}

\begin{proof}
The target domain error $\epsilon_T(h_y)$ can be decomposed as source error plus distribution shift and a remainder:
\begin{align}
\epsilon_T(h_y) &= \epsilon_T(h_y) + \epsilon_S(h_y) - \epsilon_S(h_y) + \epsilon_S(h_y, f_T) - \epsilon_S(h_y, f_T).
\end{align}

Rearranging terms, we have:
\begin{equation}
\begin{aligned}
\epsilon_T(h_y) &\leq \epsilon_S(h_y) + |\epsilon_S(h_y, f_T) - \epsilon_S(h_y, f_S)| \\ &+ |\epsilon_T(h_y, f_T) - \epsilon_S(h_y, f_T)|. 
\end{aligned}
\end{equation}

Using the definition of error and introducing the density functions $\phi_S$ and $\phi_T$ for $D_S$ and $D_T$, respectively, the second term can be bounded as:
\begin{align}
|\epsilon_S(h_y, f_T) - \epsilon_S(h_y, f_S)| &\leq E_{D_S}[|f_S(\mathbf{x}) - f_T(\mathbf{x})|],
\end{align}
while the third term can be expressed as:
\begin{equation}
\begin{aligned}
|\epsilon_T(h_y, f_T) - \epsilon_S(h_y, f_T)| &\leq \int |\phi_S(\mathbf{x}) - \phi_T(\mathbf{x})| |h_y(\mathbf{x}) - f_T(\mathbf{x})| d\mathbf{x} \\
&\leq d(D_S, D_T).
\end{aligned}
\end{equation}

Combining these results, we obtain:
\begin{align}
\epsilon_T(h_y) \leq \epsilon_S(h_y) + E_{D_S}[|f_S(\mathbf{x}) - f_T(\mathbf{x})|] + d(D_S, D_T).
\end{align}

Alternatively, by symmetry, we can add and subtract $\epsilon_T(h_y, f_S)$ instead of $\epsilon_T(h_y, f_T)$, leading to a bound involving $E_{D_T}[|f_S(\mathbf{x}) - f_T(\mathbf{x})|]$. Taking the minimum of the two bounds results in ~\eqref{eq:bound}.
\end{proof}

\subsection{Practical Approximation of $d(D_S, D_T)$}
The variation divergence $d(D_S, D_T)$ is challenging to compute directly, as approximating the error of the optimal hyperplane discriminator is NP-hard. To address this, we approximate $d(D_S, D_T)$ using the $\mathcal{H}$-divergence metric~\cite{ben2010theory}, which measures the discrepancy between distributions as:
\begin{equation}
\begin{aligned}
d_\mathcal{H}(D_S, D_T) = \\ 2 \sup_{\mathcal{D} \sim \mathcal{H}} 
\Big| \mathrm{Pr}_{x \sim D_S}[\mathcal{D}(x) = 1] - \mathrm{Pr}_{x \sim D_T}[\mathcal{D}(x) = 1] \Big|.
\end{aligned}
\end{equation}

Consistent with prior work~\cite{ruder2017learning, allaway2021adversarial}, we approximate $\mathcal{H}$-divergence using the Jensen-Shannon (JS) divergence, which can be defined as:
\begin{equation}
\begin{aligned}
JS(P_{D_S} \| P_{D_T}) = \frac{1}{2} KL(P_{D_S} \| M) + \frac{1}{2} KL(P_{D_T} \| M), 
\\ M = \frac{P_{D_S} + P_{D_T}}{2},
\label{eq:js}
\end{aligned}
\end{equation}
where $KL(\cdot \| \cdot)$ denotes the KL divergence between the two distributions during CTTA.

\subsection{Empirical Validation}
To validate the theoretical analysis, we evaluate both inter-domain divergence (via $JS$ divergence) and intra-class divergence. Inspired by $k$-means clustering~\cite{macqueen1967classification}, the intra-class divergence is computed as:
\begin{align}
IC = \frac{1}{|C|} \sum_{f_i \sim C} \Big\| f_i - \frac{1}{|C|} \sum_{f_i \sim C} f_i \Big\|_2^2,
\label{eq:ic}
\end{align}
where $f_i$ are encoder output features, and $C$ is the class. 

As shown in Fig.~\ref{divergence} (a), MoASE already reduces inter-domain JS divergence across the corruption sequence relative to the source baseline, lowering the initial peak and damping subsequent oscillations. Building on this, MoASE++ further depresses the peak and shortens the recovery time, returning to a low-divergence regime and thereby reducing the area under the divergence curve, indicating greater resilience to abrupt shifts. Consistently, Fig.~\ref{divergence} (b) shows that MoASE yields tighter class clusters under early high-perturbation conditions, while MoASE++ further compresses and stabilizes these clusters throughout later domains. This progression aligns with our design: MoASE’s activation-aware decomposition and domain-aware routing first mitigate texture-driven volatility, and MoASE++ augments this with EMA-anchored on-policy reverse KL to suppress confirmation bias and accelerate stabilization. Overall, the joint reduction in inter-domain divergence and intra-class dispersion demonstrates that MoASE provides initial robustness gains, whereas MoASE++ delivers additional improvements in recovery speed, steady-state compactness, and the stability–plasticity trade-off under continual shifts.
\section{Experiments}
\label{experiment}
In this section, we evaluate the performance of our proposed MoASE++ on three image classification and one segmentation benchmarks under continual test-time adaptation scenarios. The results demonstrate the superiority of MoASE++, achieving improvements across various tasks over previous state-of-the-art methods by over 16.1\% in classification accuracy and 5.8\% in segmentation mIoU.

\subsection{Task settings and implementation}
\label{sec:4.1}

\textbf{Dataset.} We evaluate our method on three classification CTTA benchmarks: CIFAR10-to-CIFAR10C, CIFAR100-to-CIFAR100C~\cite{krizhevsky2009learning}, and ImageNet-to-ImageNet-C~\cite{hendrycks2019benchmarking}. CIFAR10-C, CIFAR100-C, and ImageNet-C are specifically designed to test the robustness of machine learning models against common real-world image corruptions and perturbations, including noise, blur, and compression. For the segmentation CTTA context~\cite{yang2023exploring,liu2023vida}, we evaluate on the Cityscapes-to-ACDC benchmark. The Cityscapes~\cite{cordts2016cityscapes} is used as the source domain, while the ACDC~\cite{sakaridis2021acdc} serves as the target domain to assess adaptation effectiveness.

\textbf{CTTA Task setting.}
Following ~\cite{wang2022continual}, for classification CTTA tasks, we adapt the pre-trained source model sequentially across fifteen target domains on CIFAR10-C, CIFAR100-C, and ImageNet-C, which exhibit the highest level of corruption severity (level 5). 
The model's online prediction capabilities are evaluated immediately after the input data is processed.
In the context of segmentation CTTA, we follow prior work~\cite{yang2023exploring,liu2023vida} and employ an off-the-shelf model pre-trained on the Cityscapes dataset. For the continual adaptation to target domains, we utilize the ACDC dataset, which comprises images captured under four distinct adverse weather conditions. To mimic real-life scenarios of environmental change, we repeatedly expose the model to these conditions (Fog→Night→Rain→Snow) over multiple cycles. \textcolor{red}{\textbf{Red}} suggests the best results while \textcolor{blue}{\textbf{blue}} suggests the second best if not specific mentioned.

\begin{table}[t]
\centering
\setlength\tabcolsep{5pt}
\caption{\label{tab:param} \textbf{Implementation details of MoASE++.} Specific hyperparameters for MoASE++. BS and HS denote the batch size and hidden size of the MoASE experts.}
\footnotesize
\resizebox{\columnwidth}{!}{
\begin{tabular}{|c|cccccc|}
\toprule
DA-OPD & BS & HS & EMA $\alpha$ & Views & Temp & Weight \\
\midrule
Cifar10-C & 40 & 48 & 0.999 & 2 & 1 & 0.5 \\
Cifar100-C & 40 & 192 & 0.998 & 1 & 2.5 & 0.1 \\
ImageNet-C & 16 & 192 & 0.995 & 1 & 2 & 0.3 \\ 
City.$\rightarrow$ACDC & 1 & 64 & 0.999 & 1 & 1.5 & 0.1 \\ 
\bottomrule
\end{tabular}
}
\vspace{-1em}
\end{table}

\textbf{Implementation Details.} 
In our CTTA experiments, we meticulously adhere to the implementation protocols established in previous research~\cite{wang2022continual} to ensure both consistency and comparability across our studies. 
For the backbone architectures in classification CTTA, we employ ViT-base~\cite{dosovitskiy2020image}, standardizing input image sizes to 384$\times$384 pixels for CIFAR-C and 224$\times$224 pixels for ImageNet-C. In segmentation CTTA, the Segformer-B5 model~\cite{xie2021segformer}, pre-trained, serves as our source model, with input dimensions reduced from 1920$\times$1080 to 960$\times$540 for processing in target domains. $\eta$ is set to 0.1. Optimization utilizes the Adam algorithm~\cite{kingma2014adam} with $(\beta_1, \beta_2) = (0.9, 0.99)$. Specific learning rates are assigned to each task: 1e-4 for classification, and 2e-4 for segmentation. For the segmentation task, the learning rate is decayed by a factor of 0.5 every 3200 iterations. In the initialization process of the MoASE, we adopt the methodology from Adaptformer~\cite{chen2022adaptformer}. The weights of the down-projection layers are initialized using Kaiming Normal initialization~\cite{he2015delving}. The detailed hyperparameter selection of MoASE++ is shown in TABLE~\ref{tab:param}. The random seed is set to 1 across all the experiments.

\begin{table*}[t]
\caption{\label{tab:imagenet}\textbf{Classification error rate(\%) for online CTTA task.} Gain(\%) represents the percentage of accuracy improvement.}
\centering
\setlength\tabcolsep{3pt}
\begin{adjustbox}{width=1\linewidth,center=\linewidth}
\begin{tabular}{c|c|ccccccccccccccc|cc}
\toprule
Method & Venue &
 \rotatebox[origin=c]{70}{Gaussian} & \rotatebox[origin=c]{70}{shot} & \rotatebox[origin=c]{70}{impulse} & \rotatebox[origin=c]{70}{defocus} & \rotatebox[origin=c]{70}{glass} & \rotatebox[origin=c]{70}{motion} & \rotatebox[origin=c]{70}{zoom} & \rotatebox[origin=c]{70}{snow} & \rotatebox[origin=c]{70}{frost} & \rotatebox[origin=c]{70}{fog}  & \rotatebox[origin=c]{70}{brightness} & \rotatebox[origin=c]{70}{contrast} & \rotatebox[origin=c]{70}{elastic\_trans} & \rotatebox[origin=c]{70}{pixelate} & \rotatebox[origin=c]{70}{jpeg}
& Mean$\downarrow$ & Gain$\uparrow$ \\
\midrule
\multicolumn{19}{c}{CIFAR10 \quad $\Rightarrow$ \quad CIFAR10-C}\\
\midrule
Source~\cite{dosovitskiy2020image} & ICLR2021 &60.1&53.2&38.3&19.9&35.5&22.6&18.6&12.1&12.7&22.8&5.3&49.7&23.6&24.7&23.1&28.2&0.0\\
TENT~\cite{DequanWangetal2021}  & CVPR2021 &57.7&56.3&29.4&16.2&35.3&16.2&12.4&11.0&11.6&14.9&4.7&22.5&15.9&29.1&19.5&23.5&+4.7\\
CoTTA~\cite{wang2022continual} & CVPR2022 &58.7&51.3&33.0&20.1&34.8&20&15.2&11.1&11.3&18.5&4.0&34.7&18.8&\textit{19.0}&17.9&24.6&+3.6\\
VDP~\cite{gan2023decorate} & AAAI2023 &57.5&49.5&31.7&21.3&35.1&19.6&15.1&10.8&10.3&18.1&4.0&27.5&18.4&22.5&19.9&24.1&+4.1\\
BECoTTA~\cite{lee2024becotta} & ICML2024 &54.6&48.1&26.5&22.1&32.8&19.7&14.9&10.1&10.2&16.3&3.9&27.2&16.4&25.7&\textit{15.4}&22.9&+5.3\\
ViDA~\cite{liu2023vida} & ICLR2024 & 52.9 & 47.9 & \textcolor{blue}{\textit{19.4}} & \textcolor{red}{\textbf{11.4}} & 31.3 & 13.3 & 7.6 & \textcolor{blue}{\textit{7.6}} & 9.9 & 12.5 & 3.8 & 26.3 & 14.4 & 33.9 & 18.2 & 20.7 & +7.5 \\
TCA~\cite{ni2025maintaining} & ICCV2025 & 41.6 & 27.9 & 22.3 & 16.0 & \textcolor{blue}{\textit{24.4}} & \textcolor{blue}{\textit{12.5}} & \textcolor{blue}{\textit{7.4}} & 11.2 & 7.7 & 11.0 & 5.5 & \textcolor{blue}{\textit{10.0}} & 20.5 & 19.4 & 23.8 & 17.5 & +10.7 \\
MoASE~\cite{zhang2026decomposing} & AAAI2026 & \textcolor{blue}{\textit{43.7}} & \textcolor{blue}{\textit{31.3}}& 25.1 & 16.5 & 28.1 & 13.8 & 9.7 & 8.3 & \textcolor{blue}{\textit{7.1}} & \textcolor{blue}{\textit{10.1}} & \textcolor{red}{\textbf{3.0}} & 12.9 & \textcolor{blue}{\textit{12.0}} & \textcolor{blue}{\textit{16.3}} & \textcolor{red}{\textbf{13.5}} & \textcolor{blue}{\textit{16.8}} & \textcolor{blue}{+\textit{11.4}} \\
\rowcolor{green!10}MoASE++ & \textbf{Proposed} & \textcolor{red}{\textbf{38.4}} & \textcolor{red}{\textbf{24.4}} & \textcolor{red}{\textbf{18.4}} & \textcolor{blue}{\textit{13.7}} & \textcolor{red}{\textbf{22.5}} & \textcolor{red}{\textbf{10.3}} & \textcolor{red}{\textbf{6.2}} & \textcolor{red}{\textbf{8.0}} & \textcolor{red}{\textbf{6.0}} & \textcolor{red}{\textbf{7.1}} & \textcolor{blue}{\textit{3.2}} & \textcolor{red}{\textbf{5.3}} & \textcolor{red}{\textbf{12.5}} & \textcolor{red}{\textbf{14.2}} & \textcolor{blue}{\textit{15.1}} & \textcolor{red}{\textbf{13.7}} & \textcolor{red}{+\textbf{14.5}} \\
\midrule
\multicolumn{19}{c}{CIFAR100 \quad $\Rightarrow$ \quad CIFAR100-C}\\
\midrule
Source~\cite{dosovitskiy2020image} & ICLR2021 &55.0&51.5&26.9&24.0&60.5&29.0&21.4&21.1&25.0&35.2&11.8&34.8&43.2&56.0&35.9&35.4&0.0\\
TENT~\cite{DequanWangetal2021}  & CVPR2021 &53.0&47.0&24.6&\textcolor{blue}{\textit{22.3}}&58.5&26.5&19.0&21.0&23.0&30.1&11.8&25.2&39.0&47.1&33.3&32.1&+3.3\\
CoTTA~\cite{wang2022continual} & CVPR2022 &55.0&51.3&25.8&24.1&59.2&28.9&21.4&21.0&24.7&34.9&11.7&31.7&40.4&55.7&35.6&34.8&+0.6\\
VDP~\cite{gan2023decorate} & AAAI2023 &54.8&51.2&25.6&24.2&59.1&28.8&21.2&20.5&23.3&33.8&\textcolor{red}{\textbf{7.5}}&\textcolor{red}{\textbf{11.7}}&32.0&51.7&35.2&32.0&+3.4\\
BECotta~\cite{lee2024becotta} & ICML2024 &53.2&45.2&23.6&22.5&51.3&23.4&22.3&21.4&21.5&31.5&16.8&23.7&32.5&43.6&34.7&31.3&+4.3\\
ViDA~\cite{liu2023vida} & ICLR2024 & 50.1 & 40.7 & 22.0 & \textcolor{red}{\textbf{21.2}} & 45.2 & \textbf{21.6} &\textcolor{red}{\textbf{16.5}} &\textcolor{red}{\textbf{17.9}} &\textcolor{red}{\textbf{16.6}} &25.6 &\textcolor{blue}{\textit{11.5}} & 29.0 & \textcolor{blue}{\textit{29.6}} &34.7 &\textcolor{red}{\textbf{27.1}} &27.3 &+8.1 \\
TCA~\cite{ni2025maintaining} & ICCV2025 & \textcolor{blue}{\textit{39.5}} & \textcolor{blue}{\textit{30.0}} & 21.5 & 25.7 & \textcolor{blue}{\textit{36.0}} & 24.2 & 19.1 & 21.7 & 20.2 & 25.2 & 15.2 & 22.1 & 34.5 & 34.1 & 36.5 & 27.1 & +8.3 \\
MoASE~\cite{zhang2026decomposing} & AAAI2026 & 42.6 & 34.2 & \textcolor{blue}{\textit{20.5}} & 23.1 & 38.7 & \textcolor{blue}{\textit{22.2}} & \textcolor{blue}{\textit{17.3}} & \textcolor{blue}{\textit{18.8}} & 18.0 & \textcolor{blue}{\textit{24.1}} & 12.7& 24.4 & \textcolor{red}{\textbf{28.2}} & \textcolor{blue}{\textit{32.7}} & \textcolor{blue}{\textit{29.0}} & \textcolor{blue}{\textit{25.8}} & \textcolor{blue}{+\textit{9.6}} \\
\rowcolor{green!10}MoASE++ & \textbf{Proposed} & \textcolor{red}{\textbf{36.9}} & \textcolor{red}{\textbf{28.3}} & \textcolor{red}{\textbf{19.8}} & 23.5 & \textcolor{red}{\textbf{35.6}} & \textcolor{red}{\textbf{21.6}} & 17.8 & 19.4 & \textcolor{blue}{\textit{17.8}} & \textcolor{red}{\textbf{23.1}} & 13.9 & \textcolor{blue}{\textit{20.3}} & 30.2 & \textcolor{red}{\textbf{30.1}} & 31.5 & \textcolor{red}{\textbf{24.7}} & \textcolor{red}{+\textbf{10.7}} \\
\midrule
\multicolumn{19}{c}{ImageNet \quad $\Rightarrow$ \quad ImageNet-C}\\
\midrule
Source~\cite{dosovitskiy2020image} & ICLR2021 &53.0&51.8&52.1&68.5&78.8&58.5&63.3&49.9&54.2&57.7&26.4&91.4&57.5&38.0&36.2&55.8&0.0\\
TENT~\cite{DequanWangetal2021}  & CVPR2021 &52.2&48.9&49.2&65.8&73&54.5&58.4&44.0&47.7&50.3&\textcolor{red}{\textbf{23.9}}&72.8&55.7&34.4&33.9&51.0&+4.8\\
CoTTA~\cite{wang2022continual} & CVPR2022 &52.9&51.6&51.4&68.3&78.1&57.1&62.0&48.2&52.7&55.3&25.9&90.0&56.4&36.4&35.2&54.8&+1.0\\
VDP~\cite{gan2023decorate} & AAAI2023 &52.7&51.6&50.1&58.1&70.2&56.1&58.1&42.1&46.1&45.8&\textcolor{blue}{\textit{23.6}}&70.4&54.9&34.5&36.1&50.0&+5.8\\
BECoTTA~\cite{lee2024becotta} & ICML2024 &50.1&46.6&42.3&57.1&65.8&51.3&51.7&42.0&41.4&42.5&25.0&67.3&50.3&\textit{31.6}&34.4&44.0&+11.8\\
ViDA~\cite{liu2023vida} & ICLR2024 & 47.7 & 42.5 & 42.9 & 52.2 & 56.9 & 45.5 & 48.9 & 38.9 & 42.7 & 40.7 & 24.3 & \textcolor{blue}{\textit{52.8}} & 49.1 & 33.5 & 33.1 & 43.4 & +12.4 \\
TCA~\cite{ni2025maintaining} & ICCV2025 & 43.9 & 39.7 & 41.7 & 59.5 & \textcolor{red}{\textbf{50.3}} & 44.8 & \textcolor{red}{\textbf{45.0}} & 43.4 & 39.8 & 39.8 & 27.0 & \textcolor{red}{\textbf{52.4}} & 44.1 & 34.0 & 35.8 & 42.7 & +13.1 \\
MoASE~\cite{zhang2026decomposing} & AAAI2026 & \textcolor{blue}{\textit{43.1}} & \textcolor{red}{\textbf{38.4}} & \textcolor{blue}{\textit{36.8}} & \textcolor{blue}{\textit{54.7}} & \textcolor{blue}{\textit{52.2}} & \textcolor{blue}{\textit{41.2}} & 48.3 & \textcolor{blue}{\textit{37.7}} & \textcolor{blue}{\textit{35.6}} & 41.1 & 25.2 & 63.5 & \textcolor{blue}{\textit{34.7}} & \textcolor{blue}{\textit{27.7}} & \textcolor{blue}{\textit{28.3}} & \textcolor{blue}{\textit{40.5}} & \textcolor{blue}{+\textit{15.3}}\\
\rowcolor{green!10}MoASE++ & \textbf{Proposed} & \textcolor{red}{\textbf{41.3}} & \textcolor{blue}{\textit{38.7}} & \textcolor{red}{\textbf{34.3}} & \textcolor{red}{\textit{50.3}} & 54.5 & \textcolor{red}{\textbf{40.0}} & \textcolor{blue}{\textit{46.2}} & \textcolor{red}{\textbf{34.7}} & \textcolor{red}{\textbf{34.2}} & \textcolor{red}{\textbf{37.0}} & 27.1 & 63.3 & \textcolor{red}{\textbf{33.4}} & \textcolor{red}{\textbf{26.4}} & \textcolor{red}{\textbf{27.8}} & \textcolor{red}{\textbf{39.7}}& \textcolor{red}{+\textbf{16.1}}\\
\bottomrule
\end{tabular}
\end{adjustbox}
\end{table*}

\textbf{Baselines.}
We compare our model with SOTA CTTA methods, including the entropy-based method TENT~\cite{wang2020tent}, the landmark CTTA work with mean-teacher framework CoTTA~\cite{wang2022continual}, the visual prompt-based method VDP~\cite{gan2023decorate} and SVDP~\cite{yang2023exploring}, the multi-rank adapter-based method ViDA~\cite{liu2023vida}, the MAE-based method Continual-MAE, which is termed as C-MAE~\cite{liu2023adaptive}, TCA~\cite{ni2025maintaining} with batch imbalance topology weighting mechanism, and the MoE-based method BECoTTA~\cite{lee2024becotta}. Noted that we report the BECoTTA results $\textit{w/o}$ Source Domain Augmentation (SDA) for a fair comparison.  

\subsection{Quantitative analysis}
\textbf{The effectiveness of classification CTTA.} As TABLE~\ref{tab:imagenet} validates the effectiveness of MoASE as we conduct experiments on CIFAR10-to-CIFAR10-C and ImageNet-to-ImageNet-C, which consists of fifteen corruption types that occur sequentially during the test time. For MoASE, the average classification error is up to 55.8\% when we directly test the source model on target domains with ImageNet-C. Our method can outperform all previous methods, achieving 15.3\% and 2.9\% improvements over the source model and the previous SOTA method, respectively. Moreover, MoASE demonstrates remarkable performance across most corruption types, highlighting its effective mitigation of error accumulation and catastrophic forgetting. Beyond MoASE, MoASE++ further reduces the online error on CIFAR10-C and ImageNet-C, achieving mean errors of 13.7\% and 39.7\%, respectively. This translates into +14.5\% and +16.1\% gains over the source model, and consistent improvements over prior SOTA under most corruptions, indicating that DA-OPD complements activation‑level disentanglement by stabilizing updates on unlabeled streams.

\textbf{The effectiveness of segmentation CTTA.} As presented in TABLE~\ref{tab:ACDC}, as we conduct experiments on the four scenarios of Cityscapes-to-ACDC and repeat three times, we observed a gradual decrease in the mIoUs of TENT over time, indicating the occurrence of catastrophic forgetting. In contrast, MoASE shows continual improvement in average mIoU (61.8→62.3→62.3) when the same sequence of target domains is repeated. Significantly, the proposed method surpasses the previous state-of-the-art CTTA method by achieving a 4.0\% increase. This notable improvement showcases MoASE's ability to adapt continuously to dynamic target domains. With DA-OPD enabled, MoASE++ achieves the highest mean mIoU with 62.5\%, and maintains the upward trend in early cycles while avoiding the late‑round degradation. The results suggest that EMA‑anchored on‑policy distillation improves the robustness–plasticity balance in long-horizon CTTA.

In addition, we present the segmentation CTTA experiment across 5 rounds in Fig.~\ref {fig:mean_iou}, showing a consistent increase in mean mIoU during the initial rounds and stable performance thereafter. First of all, we define CoTTA$^{*}$ as CoTTA with lr=2e-4. We observe that adjusting CoTTA to our setting improves early results but leads to a later drop in segmentation performance and catastrophic forgetting. In addition, MoASE achieved a 0.2\% mIoU improvement over the previous SOTA after averaging results from all rounds. Last but not least, MoASE++ consistently achieves rank‑1 performance, demonstrating its superiority in mitigating catastrophic forgetting and error accumulation.

\textbf{Adaptation across various model backbones.} 
We evaluate the flexibility of MoASE with Segformer-B0~\cite{xie2021segformer} and introduce the foundation model SAM~\cite{kirillov2023segment} as the pre-trained model for continual target domains in the Cityscapes-to-ACDC setting, following~\cite{liu2023vida}. Our method significantly enhanced performance in dynamic target domains, as shown in TABLE~\ref{tab:CIFAR10-to-CIFAR10c-dino}, achieving improvements of 0.6\% and 0.4\% for Segformer-B0 and SAM-SETR, respectively. These findings confirm that the MoASE supports effective transfer learning across model sizes and is well-suited for a variety of real-world applications in resource-limited environments. Moreover, MoASE++ yields additional gains over MoASE, e.g., +0.3 and +0.6 mIoU, while preserving efficiency. These results indicate that the proposed MoASE++ generalizes across model scales and pretraining regimes.

\begin{table*}[t]
\caption{\label{tab:ACDC} \textbf{Performance comparison for Cityscape-to-ACDC CTTA.} We repeat the target domains three times.}
\centering
\setlength\tabcolsep{2pt}
\begin{adjustbox}{width=1\linewidth,center=\linewidth}
\begin{tabular}{c|c|ccccc|ccccc|ccccc|c|c }
\toprule
\multicolumn{2}{c|}{Time}     & \multicolumn{15}{c}{$t$ \makebox[10cm]{\rightarrowfill} }                                                    \\ 
\midrule
\multicolumn{2}{c|}{Round}          & \multicolumn{5}{c|}{1}    & \multicolumn{5}{c|}{2}     & \multicolumn{5}{c|}{3}  & \multirow{2}{*}{Mean$\uparrow$}   & \multirow{2}{*}{Gain$\uparrow$}  \\ \cmidrule{1-17}
Method & Venue & Fog & Night & Rain & Snow & Mean$\uparrow$ & Fog & Night & Rain & Snow  & Mean$\uparrow$ & Fog & Night & Rain & Snow & Mean$\uparrow$ & \\ \midrule
Source\cite{xie2021segformer}   &   NIPS2021 &69.1&40.3&59.7&57.8&56.7&69.1&40.3&59.7& 	57.8&56.7&69.1&40.3&59.7& 57.8&56.7&56.7&0.0\\
TENT\cite{DequanWangetal2021}  & ICLR2021 &69.0&40.2&60.1&57.3&56.7&68.3&39.0&60.1& 	56.3&55.9&67.5&37.8&59.6&55.0&55.0&55.7&-1.0\\ 
CoTTA\cite{wang2022continual}  & CVPR2022  &70.9&41.2&62.4&59.7&58.6&70.9&41.1&62.6& 	59.7&58.6&70.9&41.0&62.7&59.7&58.6&58.6&+1.9\\ 
VDP\cite{gan2023decorate}  & AAAI2023  &70.5&41.1&62.1&59.5&  58.3    &70.4&41.1&62.2&59.4& 58.2     & 70.4&41.0&62.2&59.4& 58.2   &  58.2 & +1.5\\
BECoTTA\cite{lee2024becotta}  & ICML2024  & 72.3 & 42.0 & 63.5 & 60.1 & 59.5 & 72.4 & 41.9 & 63.5 & 60.2 & 59.5 & 72.3 & 41.9 & 63.6 & 60.2 & 59.5 & 59.5 & +2.8 \\
ViDA\cite{liu2023vida} & ICLR2024 & 71.6& 43.2& 66.0& 63.4& 61.1& 
  73.2& 44.5& 67.0& 63.9 & 62.2 
 & 73.2& 44.6 & 67.2& \textcolor{red}{\textbf{64.2}}& \textcolor{blue}{\textit{62.3}} & 61.9 
 & +5.2\\
C-MAE\cite{liu2023adaptive} & CVPR2024 & 71.9 & \textcolor{blue}{\textit{44.6}} & \textcolor{red}{\textbf{67.4}} &63.2 & \textcolor{blue}{\textit{61.8}} &71.7 & \textcolor{blue}{\textit{44.9}} &66.5 &63.1 &61.6 &72.3 & 45.4 &67.1 &63.1 &62.0 &61.8 &+5.1\\
TCA~\cite{ni2025maintaining} & ICCV2025 & 71.9& 43.6& 66.5& 63.6& 61.4& 
  \textcolor{red}{\textbf{73.5}} & 44.3& \textcolor{red}{\textbf{67.6}} & 63.7 & \textcolor{blue}{\textit{62.3}} 
 & 73.1& 44.4 & 67.3& 64.0 & 62.2 & 62.0 
 & +5.3\\
MoASE~\cite{zhang2026decomposing} & AAAI2026 & \textcolor{blue}{\textit{72.4}} & 44.5 & 66.4 & \textcolor{blue}{\textit{63.8}} & \textcolor{blue}{\textit{61.8}} & 73.0 & \textcolor{red}{\textbf{45.1}} & \textcolor{blue}{\textit{67.5}} & 63.5 & \textcolor{blue}{\textit{62.3}} & \textcolor{red}{\textbf{73.5}} & \textcolor{blue}{\textit{44.5}} & \textcolor{blue}{\textit{67.4}} & 63.5 & \textcolor{blue}{\textit{62.3}} & \textcolor{blue}{\textit{62.2}} & +\textcolor{blue}{\textit{5.5}}\\
 \rowcolor{green!10}\textbf{MoASE++} & \textbf{Proposed} & \textcolor{red}{\textbf{73.0}} & \textcolor{red}{\textbf{45.0}} & \textcolor{blue}{\textit{67.0}} & \textcolor{red}{\textbf{64.1}} & \textcolor{red}{\textbf{62.3}} & \textcolor{blue}{\textit{73.4}} & 44.8 & 67.4 & \textcolor{red}{\textbf{64.2}} & \textcolor{red}{\textbf{62.5}} & \textcolor{blue}{\textit{73.4}} & \textcolor{red}{\textbf{45.2}} & \textcolor{red}{\textbf{67.6}} & \textcolor{blue}{\textit{64.1}} & \textcolor{red}{\textbf{62.6}} & \textcolor{red}{\textbf{62.5}} 
 & +\textcolor{red}{\textbf{5.8}} \\
 \bottomrule
\end{tabular}
\end{adjustbox}
\label{tab:CTTA}
\end{table*}

\begin{figure}[t]
\centering
\includegraphics[width=0.95\linewidth]{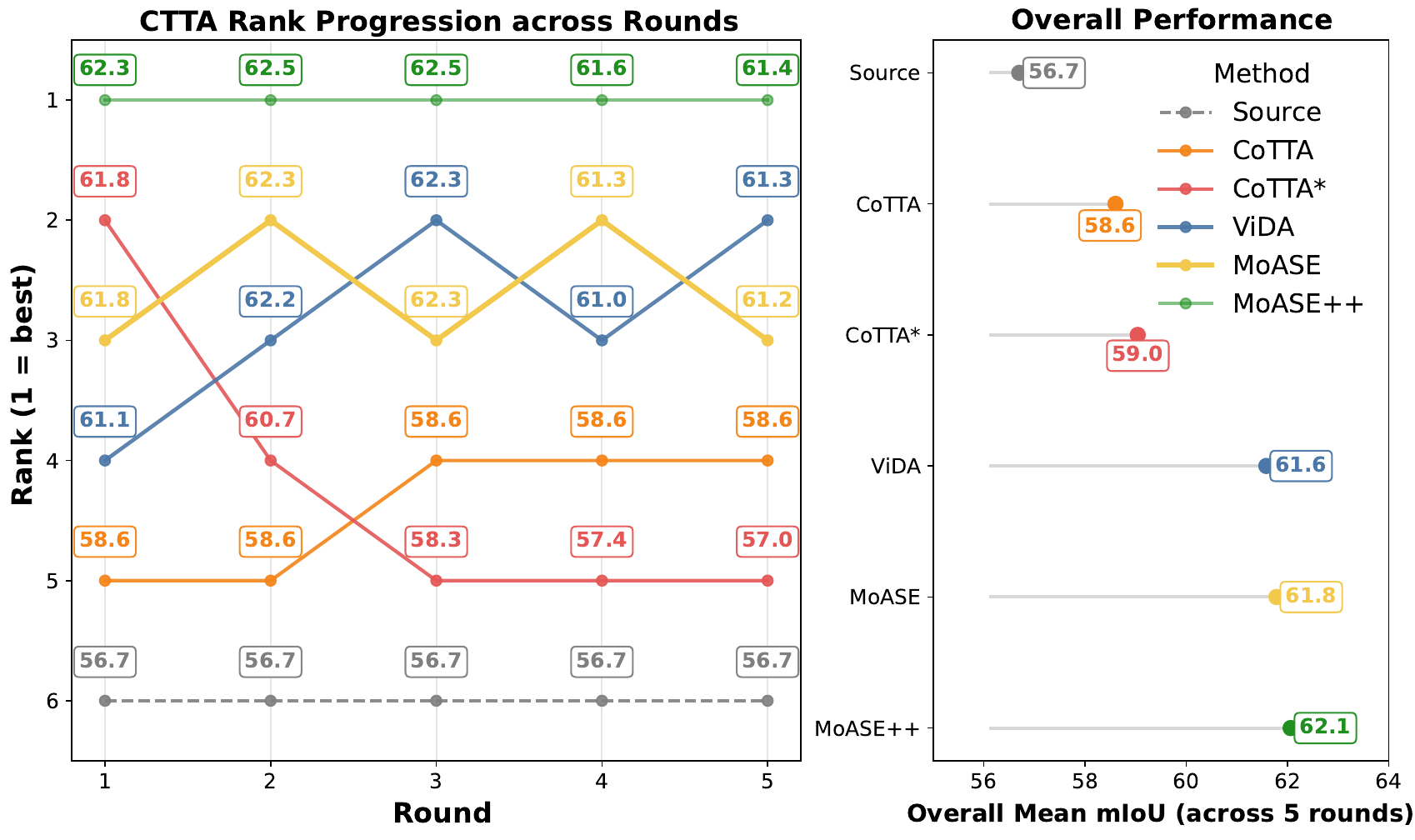}
\caption{\label{fig:mean_iou} \textbf{5 rounds segmentation CTTA on Cityscape-to-ACDC.} We sequentially repeat the same sequence of target domains 5 times with a different number of experts.}
\end{figure}


\textbf{Exploration on domain generalization (DG).} 
To evaluate the DG capabilities of our method, we adhere to the leave-one-domain-out principle, as in previous work~\cite{zhou2021domain, li2017deeper}. We select a subset of ImageNet-C domains as new source domains for model updating while designating the remaining domains as target domains where no adaptation occurs. This approach diverges from earlier domain generalization experiments, in which we implemented an unsupervised CTTA approach that updates the model solely on these unlabeled source domains. The initial model weights are derived exclusively from ImageNet pre-trained parameters. We specify that 5 of 15 ImageNet-C domains are used as source domains, with the remaining 10 domains serving as unseen target domains. The results, as detailed in Table~\ref{tab:dg10}, are noteworthy. Our method achieves a significant 12.3\% reduction in average error across these unseen domains. These encouraging results validate the DG capabilities of our approach, demonstrating its effectiveness in extracting domain-shared knowledge of MoASE. In addition, MoASE++ achieves a 39.8\% mean error on unseen domains, a +13.5\% improvement over the source model. The strong zero‑update transfer suggests that disentangling domain‑agnostic structure, coupled with on‑policy distillation, yields more durable invariances. This success offers a novel perspective on enhancing DG performance within an unsupervised framework.

\begin{table}[t]
\centering
\setlength\tabcolsep{3pt}
\caption{\label{tab:CIFAR10-to-CIFAR10c-dino} \textbf{Different size of backbone and foundation model}. mIoU score for segmentation.}
\footnotesize
\resizebox{\columnwidth}{!}{
\begin{tabular}{c|c|ccccc}
\toprule
Backbone & Method & Fog & Night & Rain & Snow & Mean$\uparrow$ \\\midrule
\multirow{3}{*}{Segformer-B0~\cite{xie2021segformer}} 
& ViDA~\cite{liu2023vida} & 57.9 & \textcolor{blue}{\textit{27.8}} & 53.1 & 51.6 & 47.6 \\
& MoASE~\cite{zhang2026decomposing} & \textcolor{blue}{\textit{58.2}} & \textcolor{red}{\textbf{28.7}} & \textcolor{blue}{\textit{53.6}} & \textcolor{blue}{\textit{52.2}} & \textcolor{blue}{\textit{48.2}} \\
& \cellcolor{green!10}MoASE++ & \cellcolor{green!10}\textcolor{red}{\textbf{59.0}} & \cellcolor{green!10}27.6 & \cellcolor{green!10}\textcolor{red}{\textbf{54.7}} & \cellcolor{green!10}\textcolor{red}{\textbf{52.8}} & \cellcolor{green!10}\textcolor{red}{\textbf{48.5}} \\
\midrule
\multirow{3}{*}{SAM~\cite{kirillov2023segment}-SETR~\cite{zheng2021rethinking}} 
& ViDA~\cite{liu2023vida} & 76.5 & 47.2 & 68.1 & 70.7 & 65.6 \\
& MoASE~\cite{zhang2026decomposing} & \textcolor{blue}{\textit{76.8}} & \textcolor{blue}{\textit{47.6}} & \textcolor{red}{\textbf{68.7}} & \textcolor{blue}{\textit{71.0}} & \textcolor{blue}{\textit{66.0}} \\
& \cellcolor{green!10}MoASE++ & \cellcolor{green!10}\textcolor{red}{\textbf{77.2}} & \cellcolor{green!10}\textcolor{red}{\textbf{48.8}} & \cellcolor{green!10}\textcolor{blue}{\textit{68.4}} & \cellcolor{green!10}\textcolor{red}{\textbf{71.8}} & \cellcolor{green!10}\textcolor{red}{\textbf{66.6}} \\
\bottomrule
\end{tabular}
}
\end{table}

\subsection{Computation analysis.}
\label{ap:comp}
A primary concern in CTTA is computational efficiency under mobile/edge constraints. From Fig.~\ref{fig:costs}, MoASE sits near the Pareto front, whereas MoASE++ consistently occupies the upper‑right in all three panels, which delivers the best accuracy but also the highest overhead: FLOPs rise to roughly 850 GMac vs. MoASE 677.31GMac, while memory usage and parameter count are only slightly higher than other baselines, accompanied by an accuracy of 62.5\%. This overhead mainly stems from multi‑expert routing and the DA‑OPD alignment. As future work, we envision an edge–cloud co‑inference scheme that keeps lightweight inference on the edge while scheduling heavier full‑expert updates in the cloud to amortize peak costs and reduce end‑to‑end latency.

\begin{table*}[t]
\centering
 \caption{\label{tab:dg10} \textbf{The domain generalization experiments on ImageNet-C}, where the source model was continually adapted on the first 5 domains and directly tested on 10 unseen domains. The evaluation of the results was conducted using ViT-base.}
\small
\setlength\tabcolsep{8pt}
\resizebox{1.99\columnwidth}{!}{%
\begin{tabular}{c|c|cccccccccc|cc}
\toprule
 & &  \multicolumn{10}{c|}{\textbf{Directly test on 10 unseen domains}}& \multicolumn{2}{c}{\textbf{Unseen}} \\ 
 \midrule
Method & Venue & \rotatebox[origin=c]{70}{motion} & \rotatebox[origin=c]{70}{zoom} & \rotatebox[origin=c]{70}{snow} & \rotatebox[origin=c]{70}{frost} & \rotatebox[origin=c]{70}{fog} & \rotatebox[origin=c]{70}{bright.} & \rotatebox[origin=c]{70}{contrast} & \rotatebox[origin=c]{70}{elastic} & \rotatebox[origin=c]{70}{pixelate} & \rotatebox[origin=c]{70}{jpeg}
& Mean$\downarrow$ & Gains$\uparrow$ \\
\midrule
Source~\cite{xie2021segformer}& ICLR2021 & 58.5&63.3&49.9&54.2&57.7&26.4&91.4&57.5&38.0&36.2&53.3&0.0\\
TENT~\cite{DequanWangetal2021}& CVPR2021& 56.0 & 61.3& 45.7 &49.6   &  56.6 &24.8&94.0&55.6&37.1&35.1&51.6&+1.7\\
CoTTA~\cite{wang2022continual}& CVPR2022& 57.3 & 62.1& 49.1 &52.0   &  57.1 &26.4&91.9&57.1&37.6&35.3&52.6&+0.7\\
ViDA~\cite{liu2023vida}& ICLR2024 &46.4 & \textcolor{blue}{\textit{52.7}} &39.8  & 43.7  & \textcolor{red}{\textbf{42.2}}&\textcolor{blue}{\textit{23.5}}&71.5&49.6&33.9&33.3&43.7&+9.6\\ 
MoASE~\cite{zhang2026decomposing}& AAAI2026& \textcolor{blue}{\textit{43.9}} & 53.7 & \textcolor{blue}{\textit{38.0}} & \textcolor{blue}{\textit{37.6}} & 46.0 & 24.0 & \textcolor{blue}{\textit{66.0}} & \textcolor{blue}{\textit{42.8}} & \textcolor{blue}{\textit{28.9}} & \textcolor{blue}{\textit{29.6}} & \textcolor{blue}{\textit{41.0}} & \textcolor{blue}{\textit{+12.3}}\\ 
\rowcolor{green!10}\textbf{MoASE++}& \textbf{Proposed} & \textcolor{red}{\textbf{42.4}} & \textcolor{red}{\textbf{52.1}} & \textcolor{red}{\textbf{37.5}} & \textcolor{red}{\textbf{39.1}} & \textcolor{blue}{\textit{42.6}} & \textcolor{red}{\textit{22.6}} & \textcolor{red}{\textbf{63.3}} & \textcolor{red}{\textbf{42.9}} & \textcolor{red}{\textbf{28.2}} & \textcolor{red}{\textbf{27.3}} & \textcolor{red}{\textbf{39.8}} & \textcolor{red}{\textbf{+13.5}}\\ 
\bottomrule
\end{tabular}
}
\end{table*}

\begin{figure*}[t]
    \centering
    \subfloat[FLOPs (GMac)]{\includegraphics[width=0.32\linewidth]{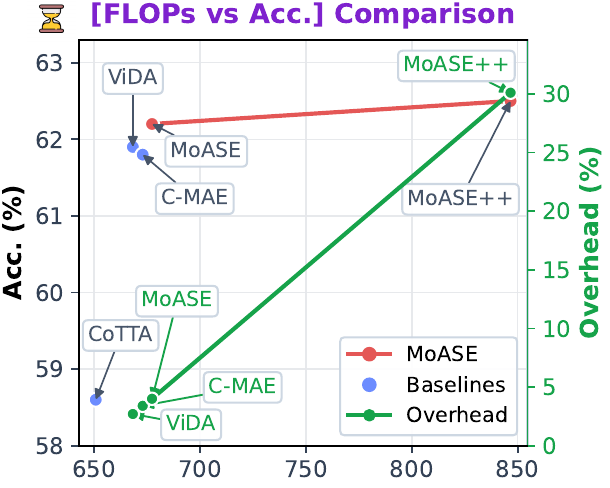}}\hfill
    \subfloat[GPU Memory (M)]{\includegraphics[width=0.32\linewidth]{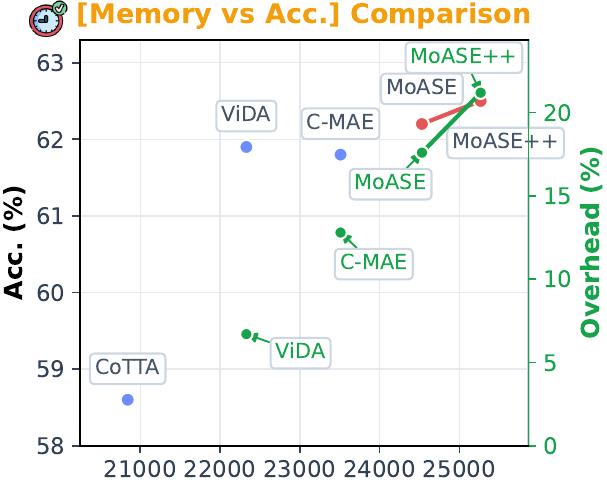}}\hfill
    \subfloat[Parameters (M)]{\includegraphics[width=0.35\linewidth]{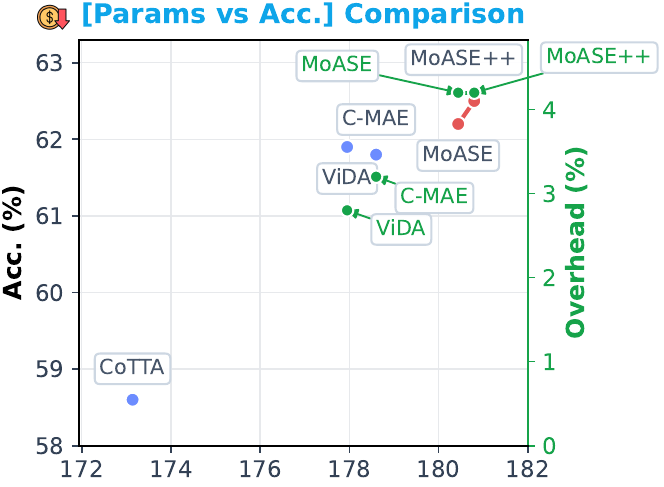}}
    \caption{\label{fig:costs} \textbf{Analysis of computational costs} for our proposed MoASE++ and previous SOTA baselines, including a detailed comparison in terms of (a) FLOPs, (b) GPU memory, and (c) the number of parameters.}
    \vspace{-1em}
\end{figure*}

\subsection{Ablation study}

\textbf{Influence of where to conduct SDD.}
\label{ap:dimension}
In our experiments, we implemented SDD on both the spatial-wise token dimension $n$ and the channel dimension $d$ of the input feature $\mathcal{F}\in \mathbb{R}^{b\times n \times d}$, as shown in Fig.~\ref{fig:token} (a). We observed that in most target domains, channel-wise MoASE maintained performance comparable to spatial-wise MoASE, with an average error rate of 19.2\% and a +9.0\% improvement over the baseline from the source domain. However, in the first two target domains, implementing SDD specifically on the spatial dimension significantly enhanced model performance, reducing the error rate to an average of over 10\%. We attribute this improvement to the distribution of domain-related information, such as styles and noise, across the entire image for each token. Consequently, decomposing the activation along the token dimension appears more effective for distinguishing domain-specific from domain-agnostic features, thereby tailoring the model more adeptly to the nuances of each domain.

\textbf{Influence of the middle-layer dimension.}
\label{ap:ml}
According to the results presented in Fig~\ref{fig:token} (b), the table compares the baseline method's performance against three experimental configurations with hidden sizes of 24, 48, 96, and 192 on the CIFAR10-to-CIFAR10-C online CTTA task. An interesting observation from the table is that the correlation between model performance and the size of the hidden dimension is not linear. Specifically, a hidden size of $h$=48 yields optimal model performance with a mean error rate of 16.8\%, whereas a smaller hidden size of $h$=24 results in a higher error rate of 20.8\%, and a larger hidden size of $h$=96 yields an even higher error rate of 21.5\%. Moreover, further increasing the hidden dimension to $h$=192 yields no significant performance improvement. The error rate is marginally reduced to 19.0\%, but this comes at the cost of a substantial increase in model size. We attribute such results to the nonlinear relationship between model architecture and task complexity, which is critical in resource-limited settings such as autonomous driving, where efficient yet effective models are essential. Therefore, we selected a different hidden size for MoASE for different datasets to best balance performance and computational efficiency. 


\begin{figure*}[t]
    \centering
    \subfloat[Where to conduct SDD?]{\includegraphics[width=0.45\linewidth]{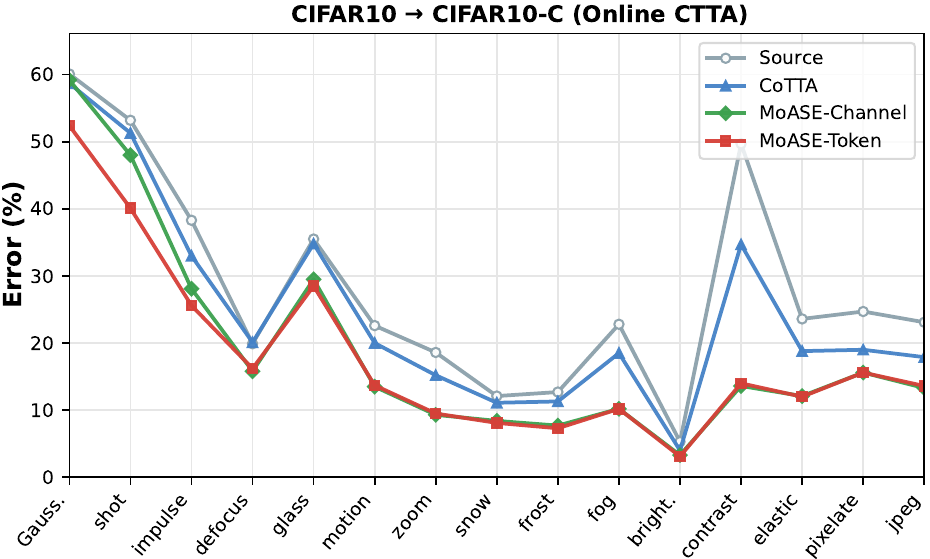}}\hfill
    \subfloat[What is the hidden dimension?]{\includegraphics[width=0.45\linewidth]{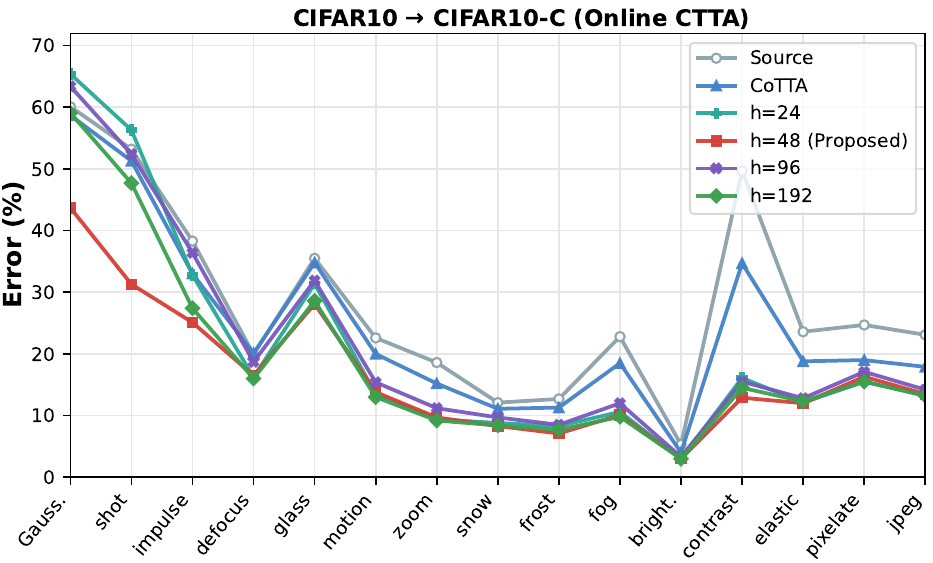}}
    \caption{\label{fig:token} \textbf{Ablation study to examine} the influence of the hyperparameters for MoASE.}
    \vspace{-1em}
\end{figure*}

\textbf{Different number of experts.}
We evaluate the impact of expert counts in MoASE on mIoU under adverse weather in Cityscape-to-ACDC, as shown in TABLE~\ref{tab:experts}. Configurations with 2 to 16 experts were tested, four experts achieved the best performance, with a mean mIoU of 61.8\% in Fog and Snow. Notably, performance does not scale linearly with the number of experts. With too few experts, the model lacks the capacity for factorization, causing high- and low-activation patterns to compete within the same expert, leading to coarse, over-smoothed boundaries. With too many experts, routing becomes sparse and noisy: each expert receives fewer effective samples per condition, calibration drifts across weather shifts, and over-specialized experts interfere during adaptation. The sweet spot around four experts yields precise activation decomposition that balances general and specific knowledge, stabilizes gating, and preserves class-consistent edges under low-visibility cues. In practice, this configuration also offers a favorable compute-to-latency trade-off, suggesting that careful control of expert granularity, rather than brute-force scaling, is key to robust CTTA under adverse weather conditions.

\begin{table}[t]
\centering
\setlength\tabcolsep{8pt}
\caption{\label{tab:experts} \textbf{Different number of experts in MoASE.} mIoU score for Cityscape-to-ACDC within the first round.}
\footnotesize
\resizebox{\columnwidth}{!}{
\begin{tabular}{c|cccc|c}
\toprule
num. E. & Fog & Night & Rain & Snow & Mean$\uparrow$ \\
\midrule
$E=2$ & \textcolor{blue}{\textit{71.6}} & 44.0 & \textcolor{red}{\textbf{66.5}} & \textcolor{blue}{\textit{63.7}} & \textcolor{blue}{\textit{61.5}} \\
\rowcolor{green!10} \textbf{$E=4$} & \textcolor{red}{\textbf{72.4}} & \textcolor{red}{\textbf{44.5}} & \textcolor{blue}{\textit{66.4}} & \textcolor{red}{\textbf{63.8}} & \textcolor{red}{\textbf{61.8}} \\
$E=8$ & 71.4 & 44.0 & 65.0 & 61.7 & 60.5 \\ 
$E=16$ & 71.5 & \textcolor{blue}{\textit{44.1}} & 65.9 & 63.3 & 61.2 \\ 
\bottomrule
\end{tabular}
}
\end{table}

\begin{table}[t]
\centering
\setlength\tabcolsep{1pt}
\caption{\label{tab:number-ds-da} \textbf{Balance of DA and DS experts.} mIoU score for Cityscape-to-ACDC with different numbers of DS and DA experts for MoASE.}
\footnotesize
\resizebox{\columnwidth}{!}{
\begin{tabular}{c|ccccc}
\toprule
Backbone & $0E_{s}$-$4E_{a}$ & $1E_{s}$-$3E_{a}$ & \cellcolor{green!10}$2E_{s}$-$2E_{a}$ & $3E_{s}$-$1E_{a}$ & $4E_{s}$-$0E_{a}$ \\
\midrule
MoASE & 61.1 & 61.7 & \cellcolor{green!10}62.2 & 60.5 & 58.6 \\
MoASE++ & 61.3 & \textcolor{blue}{\textit{62.0}} & \cellcolor{green!10}\textcolor{red}{\textbf{62.5}} & 60.7 & 59.5 \\
\bottomrule
\end{tabular}
}
\vspace{-1em}
\end{table}

\textbf{Expert balance.}
We apply balanced $E_{s}/E_{a}$ experts (e.g., $2E_{s}$-$2E_{a}$) to ensure equitable learning of high- and low-activation features, hypothesizing improved cross-domain understanding.
Beyond capacity fairness, balanced assignments reduce gradient competition and prevent collapse toward a single activation regime, yielding more stable routing and better calibration under distribution shift.
In contrast, asymmetric expert distributions in TABLE~\ref{tab:number-ds-da} consistently underperform others. For example, over-allocating $E_{s}$ overfits to salient, high-contrast cues and weakens robustness to texture or illumination corruptions, while over-allocating $E_{a}$ blurs class-discriminative structure and slows adaptation.
These results indicate that balanced experts provide complementary inductive biases, preserving object-centric signals while retaining sensitivity to low-activation context, both of which are crucial for effective CTTA and align with our hypothesis.


\begin{table}[t]
\centering
\setlength\tabcolsep{4pt}
\caption{\label{tab:ablation} \textbf{Ablation study} on Cityscape to ACDC showing the effectiveness of each module for MoASE++.}
\footnotesize
\resizebox{\columnwidth}{!}{
\begin{tabular}{c|ccccc|c}
\toprule
Methods & MoE & SDD & DAR & ASG & DA-OPD & Mean$\uparrow$ \\
\midrule
$Ex_{0}$ & \textcolor{red}{\ding{55}} & \textcolor{red}{\ding{55}} & \textcolor{red}{\ding{55}} & \textcolor{red}{\ding{55}} & \textcolor{red}{\ding{55}} & 58.6 \\ 
$Ex_{1}$ & \textcolor{rgb:red,0.0;green,0.8;blue,0.2}{\ding{51}} & \textcolor{red}{\ding{55}} & \textcolor{red}{\ding{55}} & \textcolor{red}{\ding{55}} & \textcolor{red}{\ding{55}} & 57.9 \\ 
$Ex_{2}$ & \textcolor{rgb:red,0.0;green,0.8;blue,0.2}{\ding{51}} & \textcolor{rgb:red,0.0;green,0.8;blue,0.2}{\ding{51}} & \textcolor{red}{\ding{55}} & \textcolor{red}{\ding{55}} & \textcolor{red}{\ding{55}} & 61.1 \\
$Ex_{3}$ & \textcolor{rgb:red,0.0;green,0.8;blue,0.2}{\ding{51}} & \textcolor{rgb:red,0.0;green,0.8;blue,0.2}{\ding{51}} & \textcolor{rgb:red,0.0;green,0.8;blue,0.2}{\ding{51}} & \textcolor{red}{\ding{55}} & \textcolor{red}{\ding{55}} & 61.5 \\
$Ex_{4}$ & \textcolor{rgb:red,0.0;green,0.8;blue,0.2}{\ding{51}} & \textcolor{rgb:red,0.0;green,0.8;blue,0.2}{\ding{51}} & \textcolor{rgb:red,0.0;green,0.8;blue,0.2}{\ding{51}} & \textcolor{rgb:red,0.0;green,0.8;blue,0.2}{\ding{51}} & \textcolor{red}{\ding{55}} & \textcolor{blue}{\textit{62.0}} \\
\cellcolor{green!10}$Ex_{5}$ & 
\cellcolor{green!10}\textcolor{rgb:red,0.0;green,0.8;blue,0.2}{\ding{51}} & 
\cellcolor{green!10}\textcolor{rgb:red,0.0;green,0.8;blue,0.2}{\ding{51}} & 
\cellcolor{green!10}\textcolor{rgb:red,0.0;green,0.8;blue,0.2}{\ding{51}} & 
\cellcolor{green!10}\textcolor{rgb:red,0.0;green,0.8;blue,0.2}{\ding{51}} & 
\cellcolor{green!10}\textcolor{rgb:red,0.0;green,0.8;blue,0.2}{\ding{51}} & 
\cellcolor{green!10}\textcolor{red}{\textbf{62.5}} \\
\bottomrule
\end{tabular}
}
\vspace{-1em}
\end{table}

\textbf{Effectiveness of each proposed module.} 
We conducted an ablation study in the Cityscape-to-ACDC CTTA. TABLE~\ref{tab:ablation} shows $Ex_{0}$ as the baseline using the CoTTA and $Ex_{1}$ adding a 4-expert MoE architecture to $Ex_{0}$. However, simply adding MoE did not improve performance, it instead decreased performance by 0.7\%. In contrast, implementing our SDD in $Ex_{2}$ improved segmentation results by 3.2\%. Further introductions of our modules from $Ex_{3}$ increased mIoU from 61.5\%, confirming the effectiveness of MoASE. It is noteworthy that $Ex_{4}$, employing ASG, significantly outperforms $Ex_{3}$ with full token inputs, indicating that a low activation domain-specific feature alone is sufficient to effectively determine the weights of various experts. The full MoASE++ $Ex_{5}$ attains the best performance, outperforming the intermediate variants $Ex_{4}$ without DA‑OPD. The ablations demonstrate the effectiveness of each proposed method and confirm that SDD‑based activation decomposition is the primary driver, while ASG/DAR and DA‑OPD provide complementary gains in routing stability and supervisory signal quality.

\textbf{Influence of the DA-OPD hyperparameters.} 
We also conducted an ablation study of DA-OPD hyperparameters on CIFAR100-C, as shown in TABLE~\ref{tab:hyper}, following a control-variable design to isolate the effect of each factor. Alpha sweep shows that a moderate EMA momentum is most stable, as $\alpha$=0.998 yields the lowest error at 24.7, while slightly lower or higher momentum degrades performance. Moreover, the sweep view indicates that single-view DA-OPD is preferable in this online setting, since increasing the number of views from 1→2→3 raises the error from 25.3 to 29.4 with significant computational overhead, likely due to additional variance and off-policy mismatch across views. For CIFAR‑10‑C, we hypothesize that two views are beneficial because the dataset’s relative simplicity may require additional auxiliary information to stabilize adaptation. Last but not least, the temperature sweep suggests a “moderate-softening” regime is optimal. For example, T=2.5 outperforms T=2.0 and T=3.0, as excessive or insufficient smoothing hampers stable teacher–student alignment. Finally, the weight sweep highlights the need for cautious coupling to prevent overreliance on the teacher, whereas an overly large weight amplifies confirmation bias. 

\begin{figure*}[t]
\centering
\includegraphics[width=0.99\linewidth]{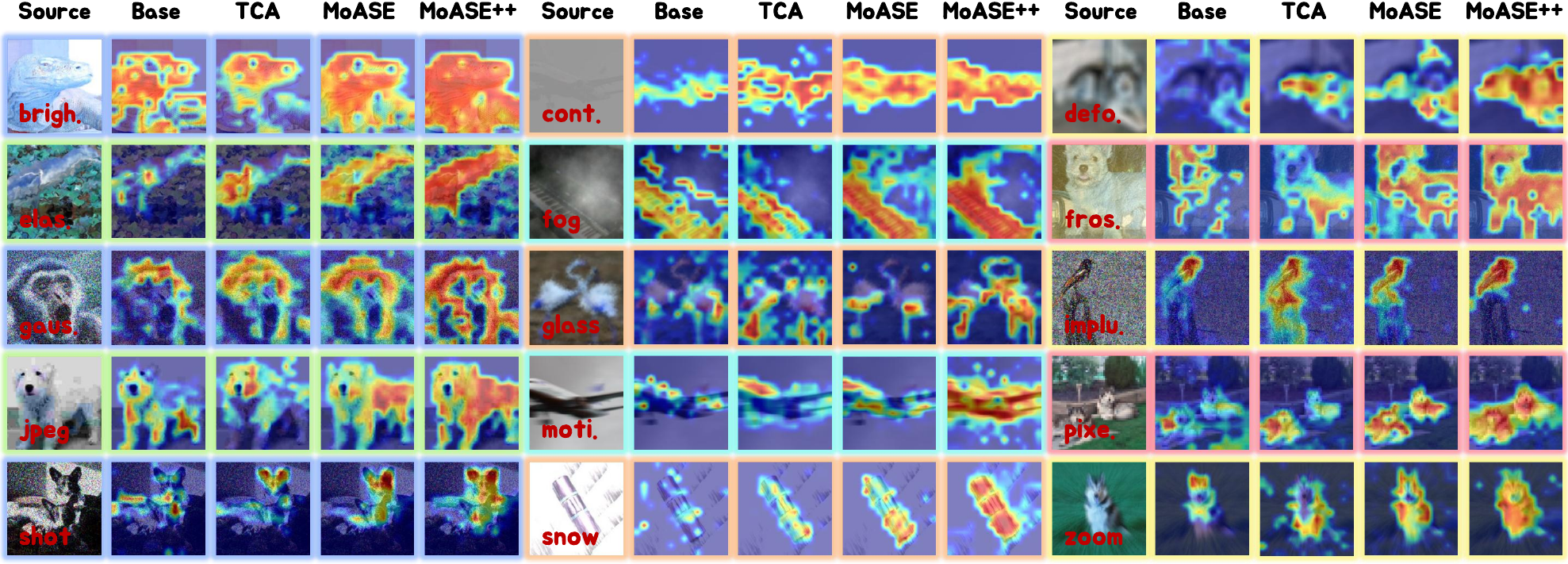}
\caption{\textbf{The qualitative analysis} of the original and CAM visualization for the source model, TCA, MoASE, and MoASE++ based on ImageNet-C with 15 domains.}
\label{fig:vis2}
\vspace{-1em}
\end{figure*}

\begin{table}[t]
\centering
\setlength\tabcolsep{1pt}
\caption{\label{tab:hyper} \textbf{Ablation study for the hyperparameters of DA-OPD} on CIFAR100-C. Groups follow the control-variable method: in each block, only one factor changes. Blue suggests our final selection for MoASE++.}
\footnotesize
\resizebox{\columnwidth}{!}{
\begin{tabular}{c|cccc|c}
\toprule
Methods & EMA $\alpha$ & OPD-views & OPD-temp & OPD-weight & Mean$\uparrow$ \\
\midrule
\multicolumn{6}{c}{\textit{Alpha sweep (fix views$=1.0$, temp$=2.5$, weight$=0.1$)}}\\
\midrule
$Ex_{5-0-0}$ & 0.997 & 1.0 & 2.5 & 0.1 & 25.1 \\
\rowcolor{green!10}$Ex_{5-0-1}$  & \textcolor{blue}{\textbf{0.998}} & 1.0 & 2.5 & 0.1 & \textcolor{red}{\textbf{24.7}} \\
$Ex_{5-0-2}$ & 0.999 & 1.0 & 2.5 & 0.1 & 25.4 \\
\midrule
\multicolumn{6}{c}{\textit{Views sweep (fix $\alpha{=}0.998$, temp$=2.0$, weight$=0.3$)}}\\
\midrule
\rowcolor{green!10}$Ex_{5-1-0}$ & 0.998 & \textcolor{blue}{\textbf{1.0}} & 2.0 & 0.3 & \textcolor{red}{\textbf{25.3}} \\
$Ex_{5-1-1}$ & 0.998 & 2.0 & 2.0 & 0.3 & 27.7 \\
$Ex_{5-1-2}$ & 0.998 & 3.0 & 2.0 & 0.3 & 29.4 \\
\midrule
\multicolumn{6}{c}{\textit{Temperature sweep (fix views$=1.0$, views$=2.5$, temp$=0.1$)}}\\
\midrule
$Ex_{5-2-1}$  & 0.998 & 1.0 & 2.0 & 0.1 & 25.0 \\
\rowcolor{green!10}$Ex_{5-2-2}$  & 0.998 & 1.0 & \textcolor{blue}{\textbf{2.5}} & 0.1 & \textcolor{red}{\textbf{24.7}} \\
$Ex_{5-2-3}$ & 0.998 & 1.0 & 3.0 & 0.1 & 25.1 \\
\midrule
\multicolumn{6}{c}{\textit{Weight sweep (fix $\alpha{=}0.998$, view$=2.0$, temp$=0.3$)}}\\
\midrule
\rowcolor{green!10}$Ex_{5-3-0}$  & 0.998 & 1.0 & 2.0 & \textcolor{blue}{\textbf{0.1}} & \textcolor{red}{\textbf{25.0}} \\
$Ex_{5-3-1}$ & 0.998 & 1.0 & 2.0 & 0.3 & 25.3 \\
$Ex_{5-3-2}$ & 0.998 & 1.0 & 2.0 & 0.5 & 29.0 \\
\bottomrule
\end{tabular}
}
\end{table}

\subsection{Qualitative analysis}
\textbf{CAM visualization.} We conducted a qualitative analysis of the CAM on ImageNet-C, as illustrated in the Fig.~\ref{fig:vis2}. We can easily observe that the original base model's attention is dispersed due to continuous domain shifts, while TCA~\cite{ni2025maintaining} also fails to detect the complete object in continual domain-adaptive scenarios. In contrast, MoASE effectively concentrates on regions relevant to the target categories, such as dogs and planes, during classification decisions. Moreover, MoASE++ tends to focus on the entire target object, reinforcing the effectiveness of our approach.

\begin{figure}[t]
\centering
\includegraphics[width=0.99\linewidth]{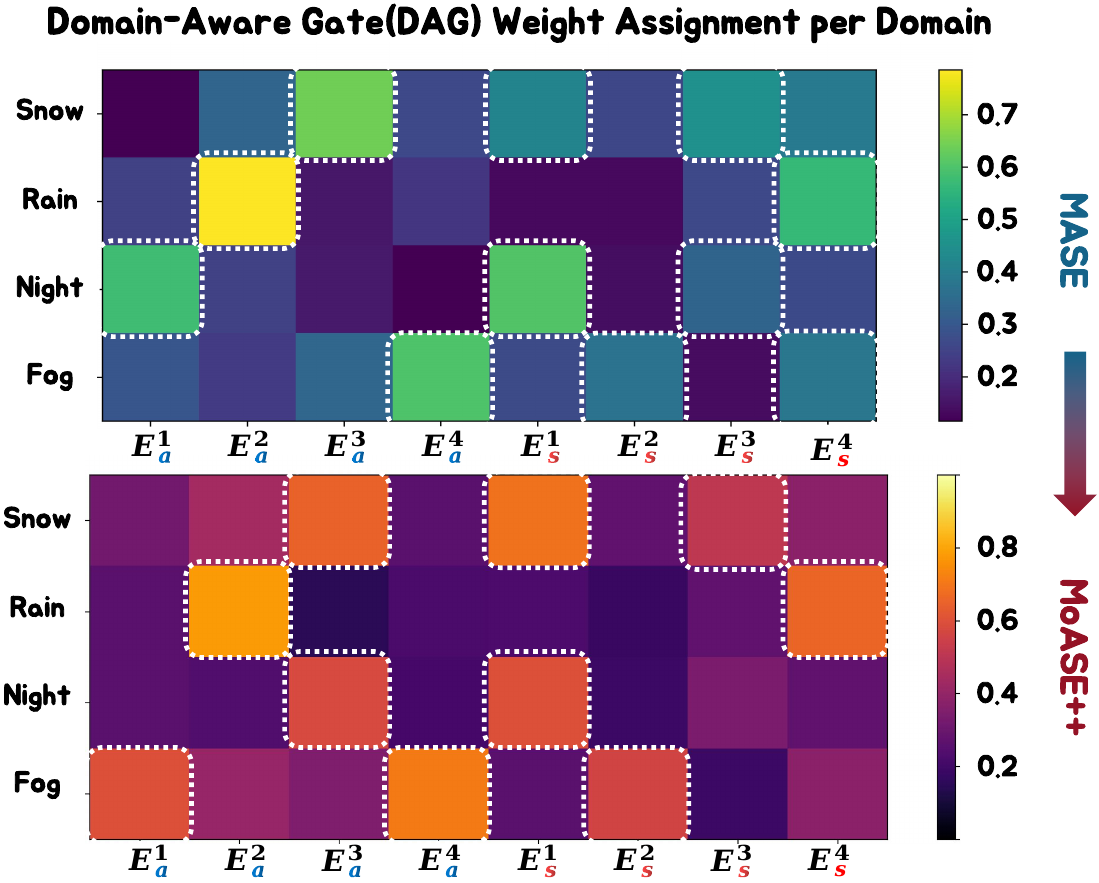}
\caption{\textbf{The visualization analysis} of the routing strategy for MoASE and MoASE++. $E_{s}^{1}$ and $E_{a}^{3}$ stands for the 1$^{st}$ domain-specific and the 3$^{rd}$ domain-agnostic experts.}
\label{fig:router}
\vspace{-1em}
\end{figure}

\textbf{Routing strategy.} We summed and normalized the output of the router on ACDC with 8 experts for better illustration, as shown in the top right of Fig.~\ref{fig:router}. DAR assigns varying weights to experts based on domain type. For instance, $E_{s}^{1}$ and $E_{s}^{3}$ primarily address snow, while $E_{s}^{2}$ and $E_{s}^{4}$ are geared towards Fog. With DA-OPD, MoASE++ shows more selective expert activation, evidencing stronger domain-specific specialization. Accordingly, DA-OPD helps DAR realize expert specialization more robustly. Moreover, each domain-agnostic expert is sensitive to different types of features, suggesting that a low-activation-only DAR can also handle characteristic token assignment.

\textbf{Segmentation visualization.} For further validation, we provide additional qualitative comparisons in the Cityscapes-to-ACDC CTTA scenario, shown in the Fig~\ref{fig:vis1}. Our method outperforms CoTTA~\cite{wang2022continual}, ViDA~\cite{liu2023vida}, and Continual-MAE~\cite{liu2023adaptive} in producing segmentation maps for the snow and night domains, precisely differentiating sidewalks from roads and identifying small objects such as people, vegetation, and fences (highlighted in the white box). These results highlight our method's precise segmentation capabilities and robustness against dynamic domain shifts. Additionally, our segmentation maps align closely with the Ground Truth, leading to significant visual improvements. In addition, MoASE++ produces more complete and class‑consistent masks than strong baselines, especially under Snow and Night, where small structures (e.g., pedestrians, fences) are better preserved. The sharper boundaries align with our design: activation‑aware experts prioritize object structure while DA‑OPD prevents drift during on‑the‑fly updates.

\begin{figure*}[t]
\centering
\includegraphics[width=0.99\linewidth]{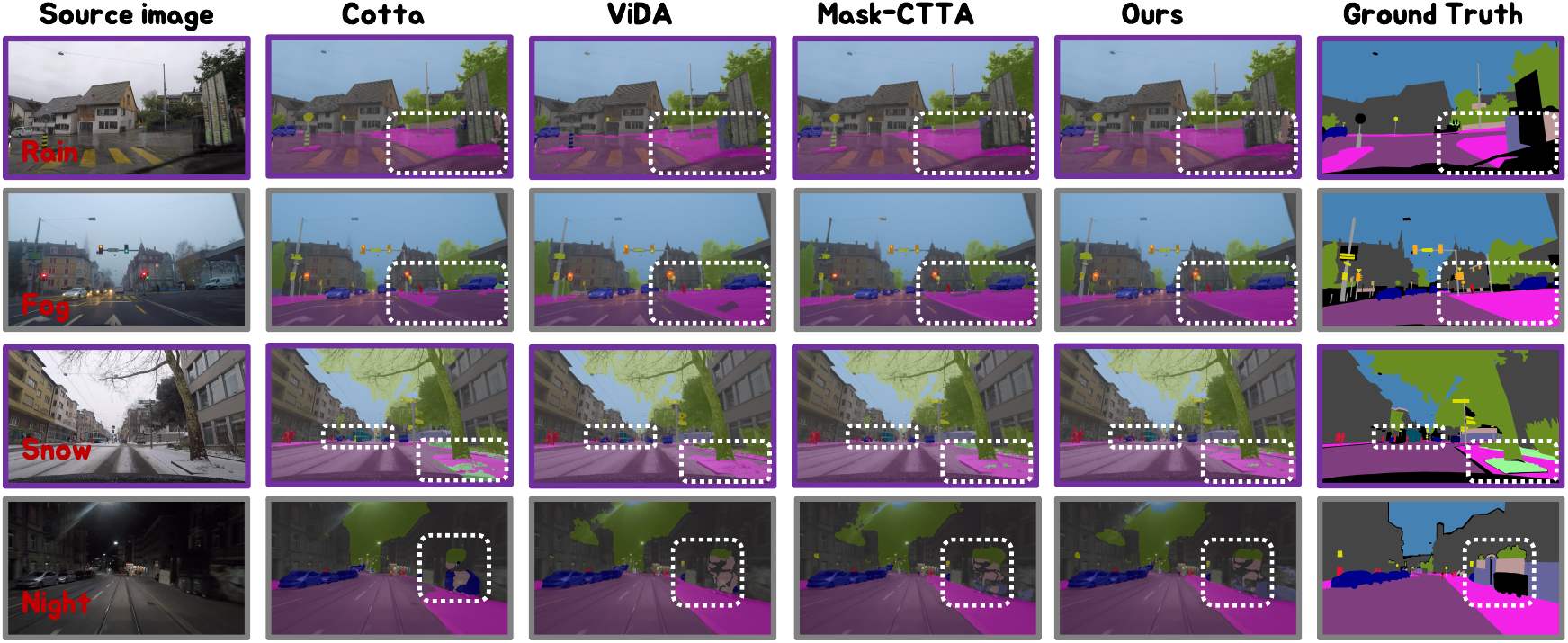}
\caption{\textbf{The qualitative analysis} of the segmentation qualitative comparison of MoASE with previous SOTA methods on the ACDC dataset. Our method could better segment different pixel-wise classes, such as shown in the white box.}
\label{fig:vis1}
\end{figure*}
\section{Conclusion}
\label{conclusion}
Inspired by the human visual system, this work studies DNNs in CTTA and introduces MoASE, a mixture-of-activation-sparsity-experts model that decouples neural activations into high- and low-level components via Spatial Differentiable Dropout. The design enhances domain-specific texture extraction while preserving domain-agnostic structure through expert modules and complementary bottlenecks, and it adaptively routes tokens with ASG and DAR. To stabilize supervision on unlabeled streams and restrain confirmation bias, we further add DA-OPD, an EMA‑anchored on‑policy reverse‑KL distillation with an augmentation policy conditioned on entropy and confidence. Together, MoASE and DA-OPD form MoASE++, which achieves consistent state-of-the-art performance on CIFAR-10/100-C, ImageNet-C, and Cityscapes→ACDC, and delivers a stronger robustness–plasticity balance under continual-shift scenarios.


\clearpage

\bibliographystyle{IEEEtran}
\bibliography{main.bib}

\newpage

\vspace{11pt}

\begin{IEEEbiography}[{\includegraphics[width=1in,height=1.25in,clip,keepaspectratio]{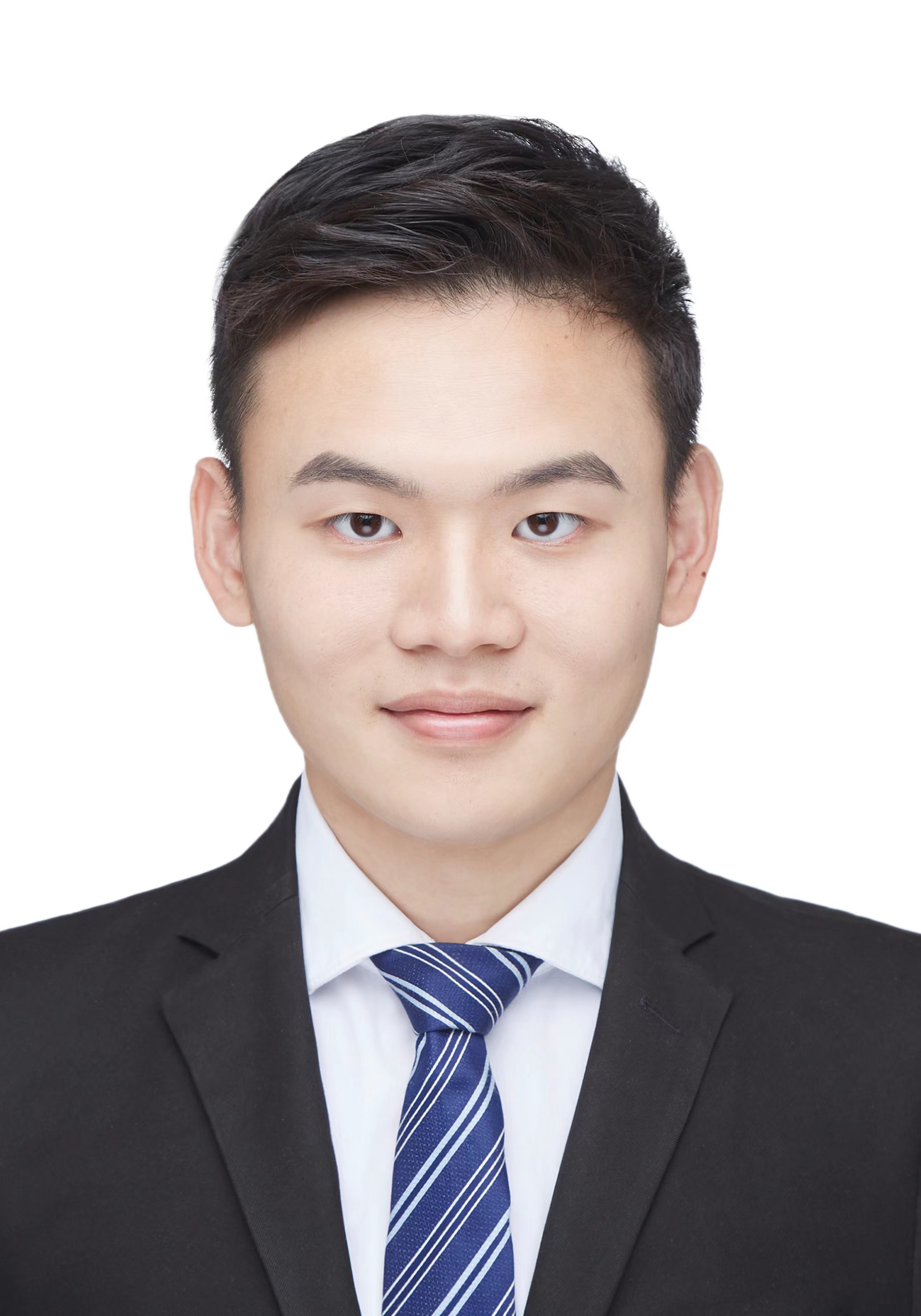}}]{Rongyu Zhang}
is a dual Ph.D. student at Nanjing University and The Hong Kong Polytechnic University. He received an M.Phil. degree at The Chinese University of Hong Kong, Shenzhen in 2023. He also received both a B.E. and a B.M. from the joint program between Beijing University of Posts and Telecommunications and Queen Mary University of London in 2021. His research interests focus on the generalization and efficient multimodal learning.

\end{IEEEbiography}

\vskip -2\baselineskip plus -1fil

\begin{IEEEbiography}[{\includegraphics[width=1in,height=1.25in,clip,keepaspectratio]{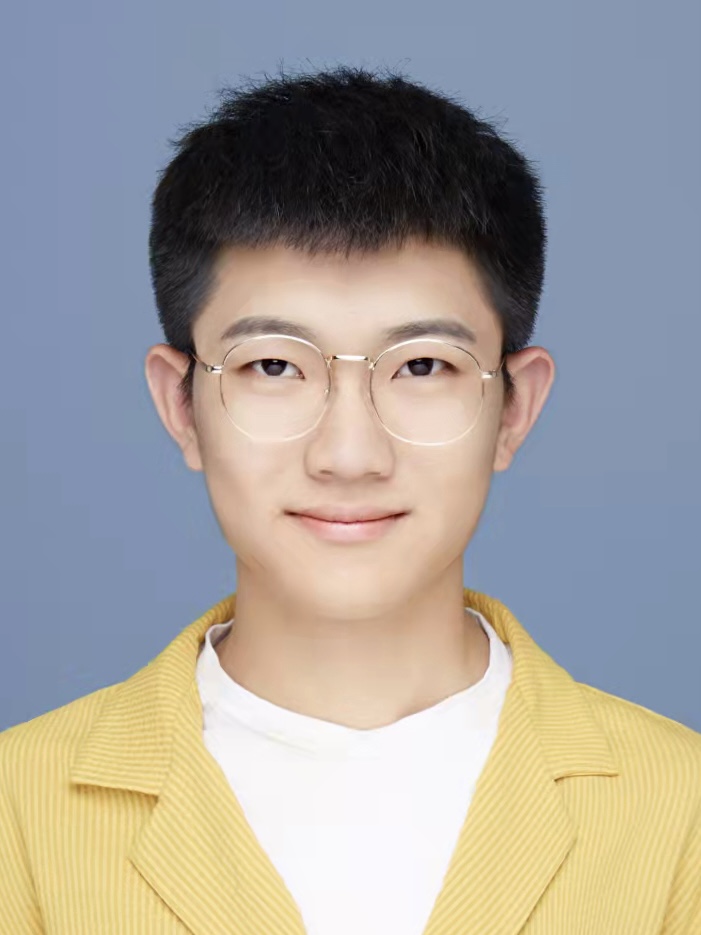}}]{Aosong Cheng}
 received his B.S. degrees in Software Engineering and Finance from Nankai University, China. He received his M.S. degree in Software Engineering from the School of Software and Microelectronics, Peking University. He is currently a Ph.D. student at the School of Artificial Intelligence, Shanghai Jiao Tong University (SJTU), China. His research interests include multimodal large models and embodied intelligence.

\end{IEEEbiography}

\vskip -2\baselineskip plus -1fil

\begin{IEEEbiography}[{\includegraphics[width=1in,height=1.25in,clip,keepaspectratio]{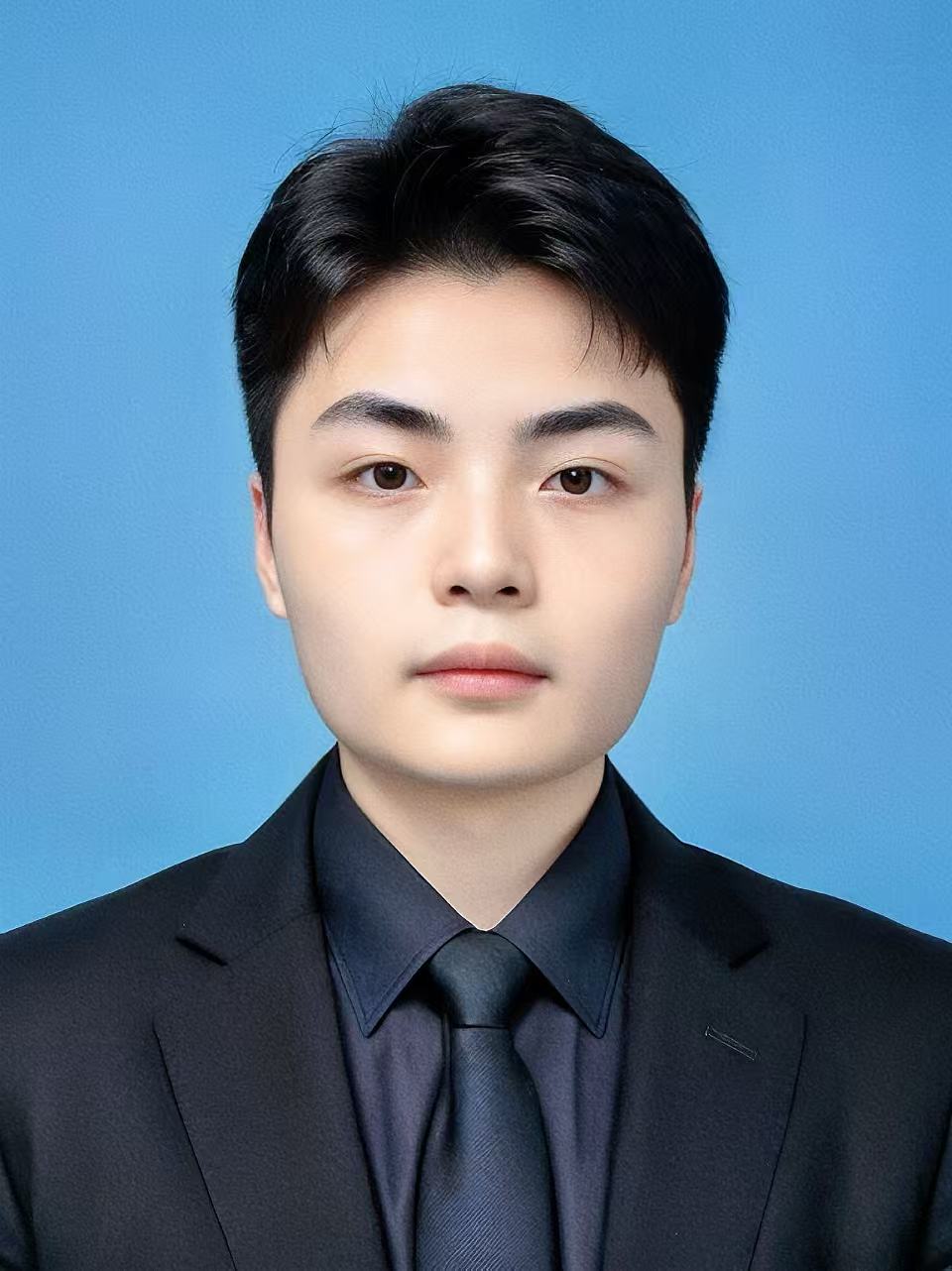}}]{Gaole Dai}
received his B.S. degree of Biology with distinction from University of Toronto (UofT), Canada in 2022. He is currently a Ph.D. student in Computer Science and Technology at Peking University (PKU), China. His research interests focus on Representation Learning, Neuroscience-inspired AI (NeuroAI), and AI for Life Science.

\end{IEEEbiography}

\vskip -2\baselineskip plus -1fil

\begin{IEEEbiography}[{\includegraphics[width=1in,height=1.25in,clip,keepaspectratio]{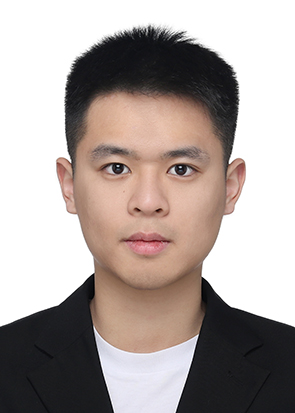}}]{Yulin Luo}
received the B.S. degree in Automation from Shanghai Jiao Tong University (SJTU), China in 2023. He is currently a Ph.D. student in Computer Science and Technology at Peking University (PKU), China. His research interests focus on Embodied AI, Multimodal Large Language Models, and Data-Centric AI.
\end{IEEEbiography}

\vskip -2\baselineskip plus -1fil

\begin{IEEEbiography}[{\includegraphics[width=1in,height=1.15in,clip,keepaspectratio]{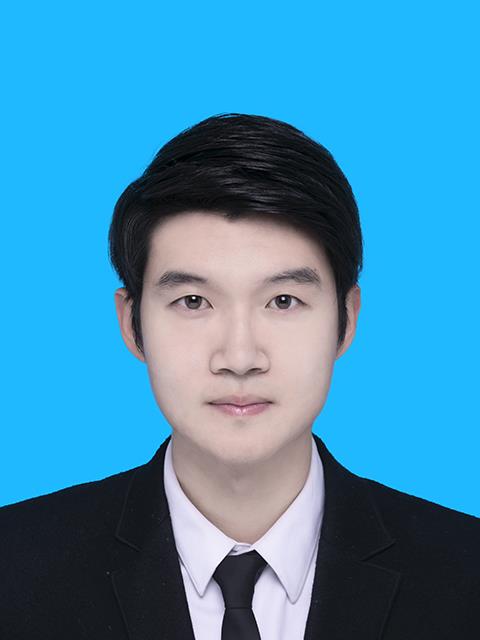}}]{Jiaming Liu}
 received the B.S. and M.S. degrees in Information and Communication Engineering from Beijing University of Posts and Telecommunications, China, in 2019 and 2022, receptively. He is currently pursuing his Ph.D. in Computer Science and Technology at Peking University. His research interests include out-of-distribution generalization, embodied AI, and multimodal scene understanding. \end{IEEEbiography}

\vskip -2\baselineskip plus -1fil

\begin{IEEEbiography}[{\includegraphics[width=1in,height=1.1in,clip,keepaspectratio]{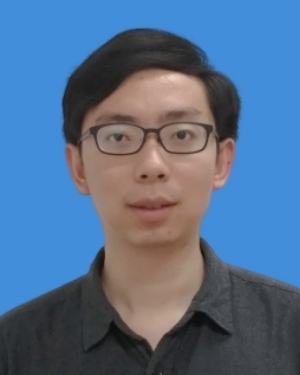}}]{Li Du}
(M’16) received his B.S degree from Southeast University, China, and his Ph.D. degree in Electrical Engineering from the University of California, Los Angeles. He is currently an associate professor in the Department of Electrical Science and Engineering at Nanjing University. His research includes analog sensing circuit design, in-memory computing design, and a high-performance AI processor for edge sensing.
\end{IEEEbiography}

\vskip -2\baselineskip plus -1fil

\begin{IEEEbiography}[{\includegraphics[width=1in,height=1.1in,clip,keepaspectratio]{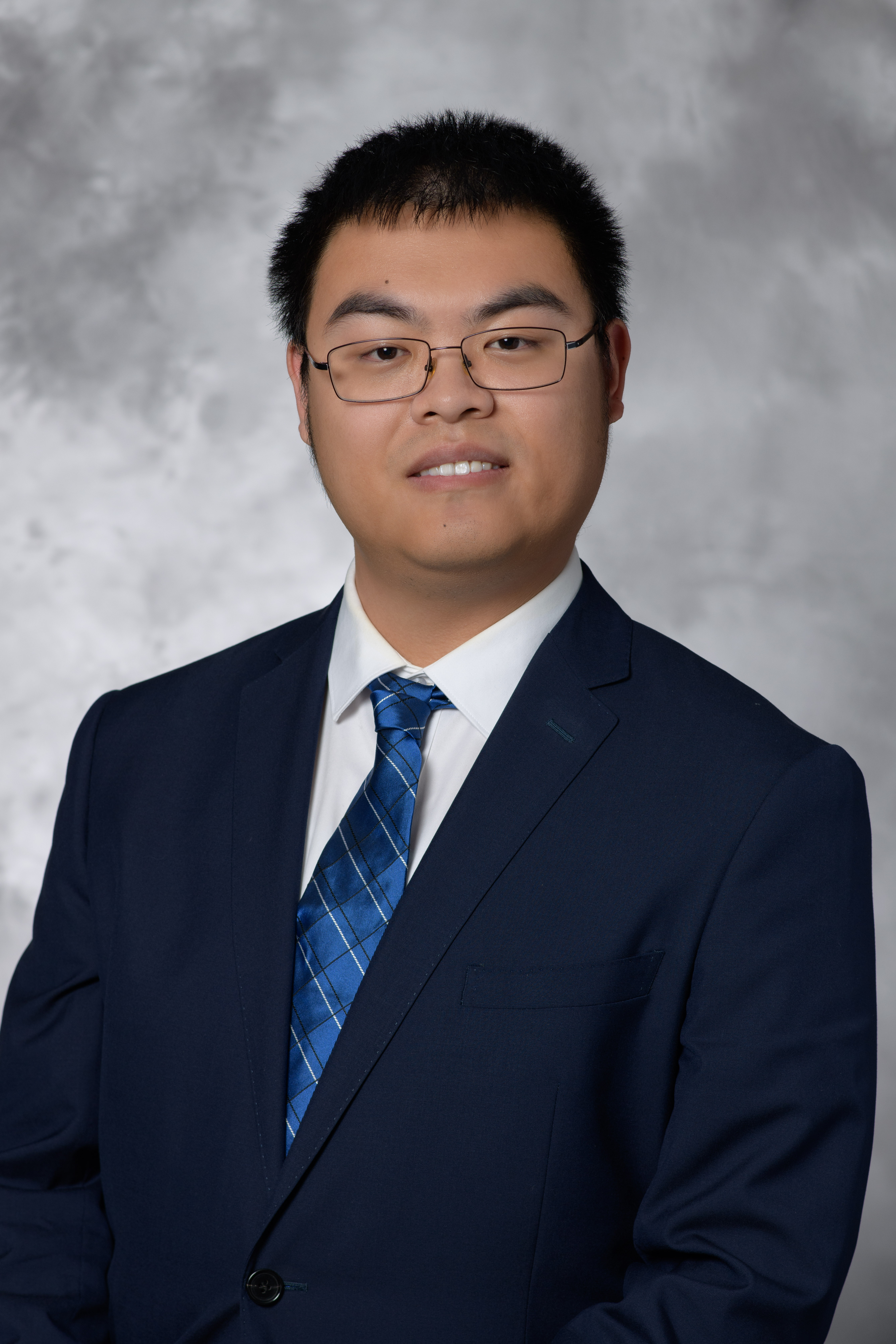}}]{Huanrui Yang} 
is an Assistant Professor in the Department of Electrical and Computer Engineering at the University of Arizona. He received his B.E degree from Tsinghua University, China, and his Ph.D. degree from Duke University. His research interest lies in the efficiency and reliability of deep learning algorithms and systems, with a recent focus on large foundation models, diffusion model, and multi-agent system.
\end{IEEEbiography}

\vskip -2\baselineskip plus -1fil

\begin{IEEEbiography}[{\includegraphics[width=1in,height=1.15in,clip,keepaspectratio]{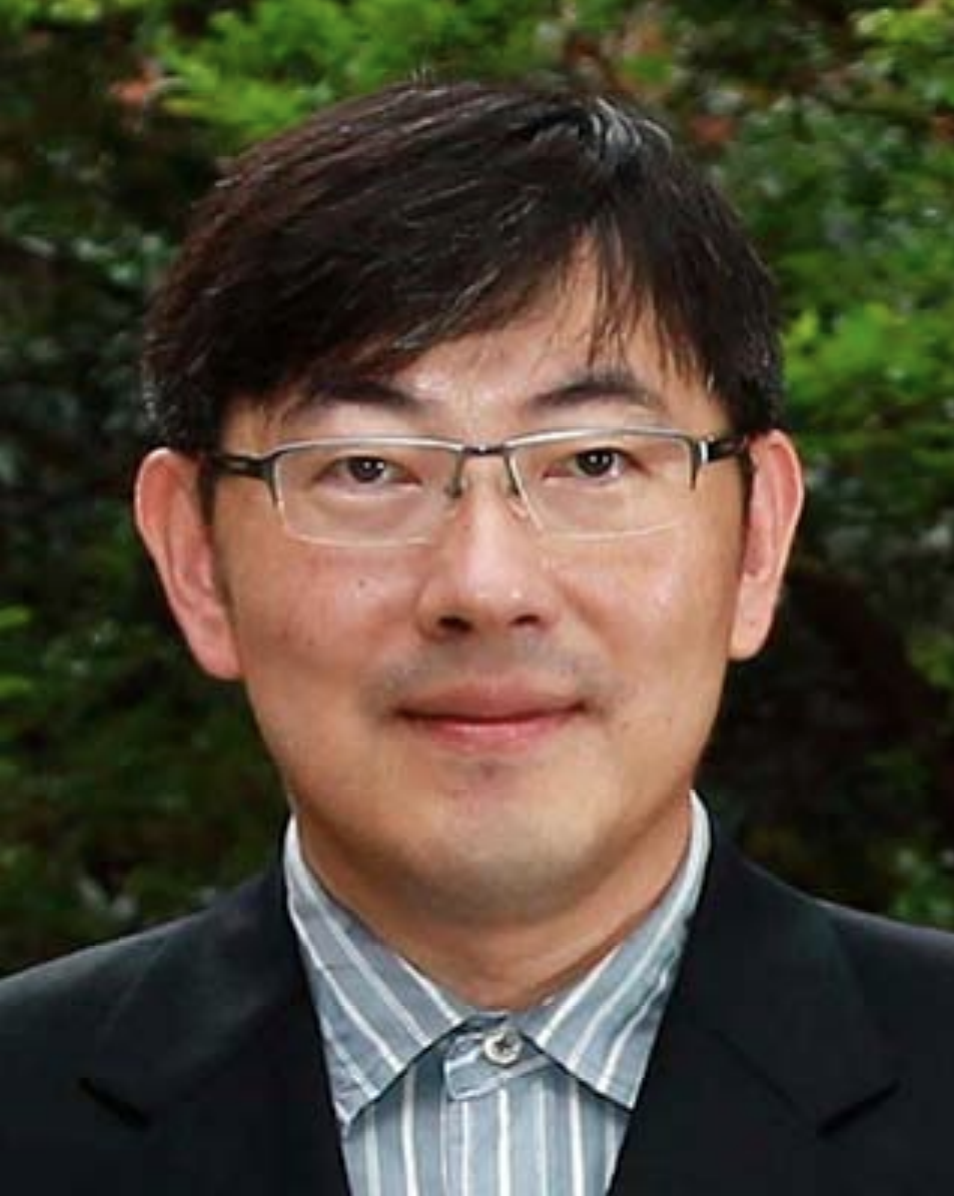}}]{Dan Wang}
is a Professor in the Department of Computing at The Hong Kong Polytechnic University. He received the Ph.D. degree from Simon Fraser University. His research falls in general computer networking and systems, where he has published in ACM SIGCOMM, ACM SIGMETRICS, the IEEE INFOCOM, and many others. He is the Steering Committee Chair of IEEE/ACM IWQoS. His research interests include network architecture and QoS, smart building, and Industry 4.0.
\end{IEEEbiography}

\vskip -2\baselineskip plus -1fil

\begin{IEEEbiography}[{\includegraphics[width=1in,height=1.25in,clip,keepaspectratio]{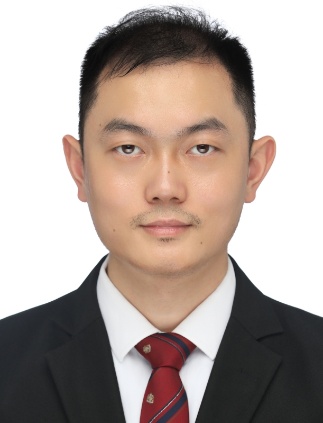}}]{Leyuan Fang} is a full professor at Hunan University. He currently serves as the Chair of the Technical Committee on Image Application and System Integration of the China Society of Image and Graphics (CSIG). He has been recognized as a Highly Cited Researcher by Clarivate Analytics and a Highly Cited Chinese Researcher by Elsevier. His research focuses on remote sensing image analysis and intelligent information processing. He has published over 200 papers in leading journals and conferences, including IEEE TPAMI, IJCV, and IEEE TIP, with more than 20,000 citations on Google Scholar. He has received multiple prestigious awards, including the 2023 IEEE GRSS Highest Impact Paper Award, and the National Natural Science Award of China and the Wu Wenjun Artificial Intelligence Science and Technology Award, and has led several major national research projects.
\end{IEEEbiography}

\vskip -2\baselineskip plus -1fil

\begin{IEEEbiography}[{\includegraphics[width=1in,height=1.1in,clip,keepaspectratio]{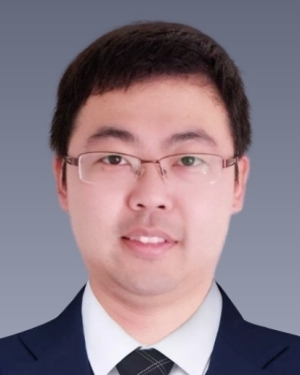}}]{Yuan Du}
(S’14-M’17-SM’21) received his B.S. degree from Southeast University (SEU), Nanjing, China, in 2009; his M.S. and Ph.D. degrees from the Electrical Engineering Department, University of California, Los Angeles (UCLA), in 2012 and 2016, respectively. Since 2019, he has been at Nanjing University in Nanjing, China, as an Associate Professor. His current research interests include designs of machine-learning hardware accelerators and RFICs.
\end{IEEEbiography}

\vskip -2\baselineskip plus -1fil

\begin{IEEEbiography}[{\includegraphics[width=1in,height=1.25in,clip,keepaspectratio]{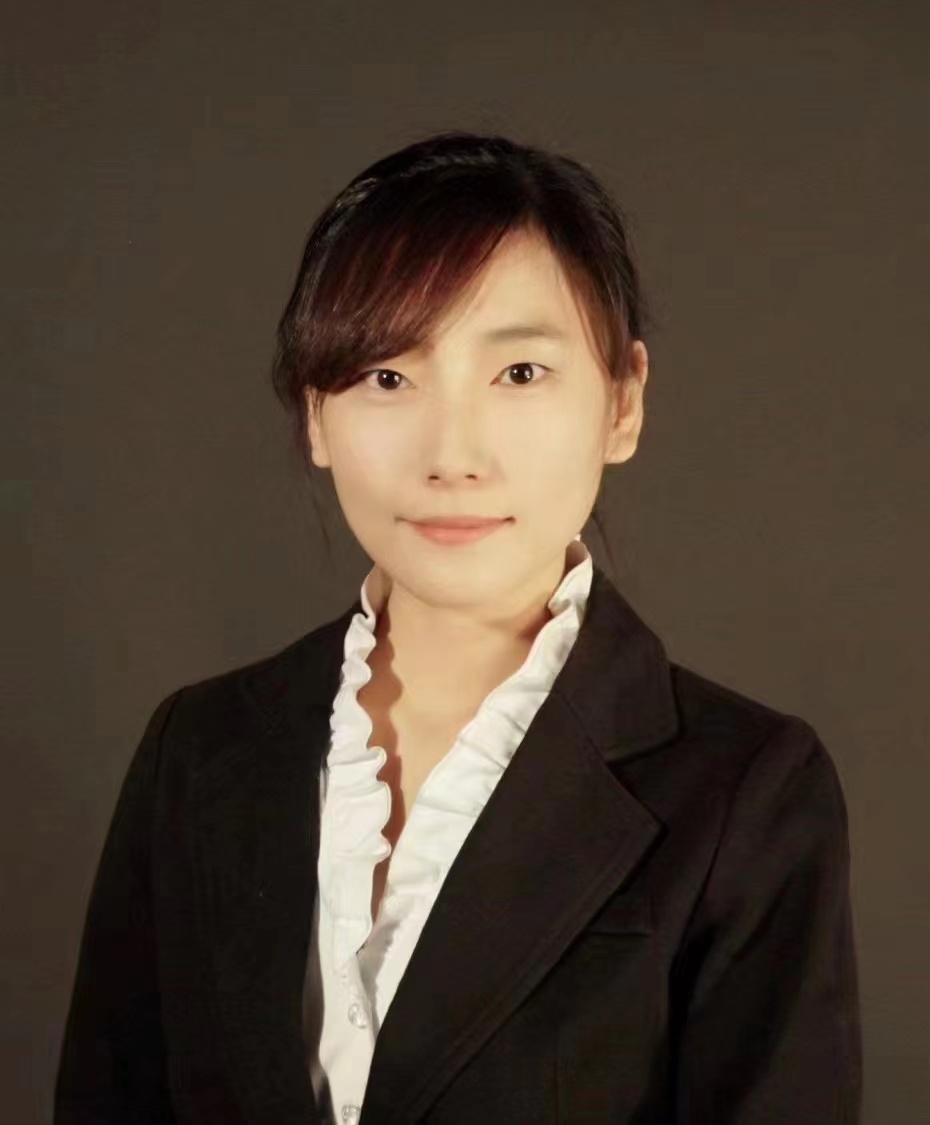}}]{Shanghang Zhang} received her Ph.D. from Carnegie Mellon University in 2018 and conducted postdoctoral research at the University of California, Berkeley. She is an Assistant Professor in the School of Computer Science at Peking University. She has published over 100 papers in top artificial intelligence journals and conferences. Her works have been cited more than 20,000 times on Google Scholar. She was honored with the Best Paper Award at the AAAI'2021. In 2018, she was recognized as an "EECS Rising Star" in the United States. She has organized workshops at top international conferences such as NeurIPS and ICML and served as a senior program committee member for AAAI 2022, 2023, and 2024.
\end{IEEEbiography}


\end{document}